\documentclass[journal]{IEEEtran}
\usepackage{graphicx} 
\usepackage{epsfig} 
\usepackage{amsmath} 
\usepackage{amssymb}  
\usepackage{subfig}
\usepackage[ruled,longend]{algorithm2e}
\usepackage{multicol}
\usepackage{xcolor}
\usepackage{amsthm}
\usepackage{siunitx}
\usepackage{threeparttable}
\usepackage{hyperref}       
\usepackage{url}            
\usepackage{booktabs}       
\usepackage{amsfonts}       
\usepackage{nicefrac}       
\usepackage{microtype}      

\usepackage{caption}
\usepackage{float}
\usepackage{comment}

\newtheorem{theorem}{Theorem}

\newtheorem{proposition}{Proposition}
\newtheorem{lemma}{Lemma}
\newtheorem{remark}{Remark}
\newcommand{\RomanNumeralCaps}[1] 
    {\MakeUppercase{\romannumeral #1}} 

\begin{document}

\title{On the Search for Feedback in Reinforcement Learning}

\author{Ran Wang, Karthikeya S. Parunandi, Aayushman Sharma, Raman Goyal, Suman Chakravorty
\thanks{The authors are with the Department of Aerospace Engineering, Texas A\&M University, College Station, TX 77843 USA. \{ rwang0417, s.parunandi, aayushmansharma, ramaniitrgoyal92, schakrav\} @tamu.edu}\\
}

\markboth{On the Search for Feedback in Reinforcement Learnings}
{Ran Wang, Karthikeya S. Parunandi, Aayushman Sharma, Raman Goyal, Suman Chakravorty}

\maketitle

\begin{abstract}
The problem of Reinforcement Learning (RL) in an unknown nonlinear dynamical system is equivalent to the search for an optimal feedback law utilizing the simulations/ rollouts of the dynamical system. Most RL techniques search over a complex global nonlinear feedback parametrization making them suffer from high training times as well as variance. Instead, we advocate searching over a local feedback representation consisting of an open-loop sequence, and an associated optimal linear feedback law completely determined by the open-loop. We show that this alternate approach results in highly efficient training, the answers obtained are repeatable and hence reliable, and the resulting closed performance is superior to global state-of-the-art RL techniques. Finally, if we replan, whenever required, which is feasible due to the fast and reliable local solution, it allows us to recover global optimality of the resulting feedback law.
\end{abstract}
\IEEEpeerreviewmaketitle
\section{Introduction}
\label{introduction}
\IEEEPARstart{T}{he} control of an unknown (stochastic) dynamical system has a rich history in the control system literature \cite{kumar2015stochastic,ioannou2012robust}. 
The stochastic adaptive control literature mostly addresses  Linear Time Invariant (LTI) problems. The optimal control of an unknown nonlinear dynamical system with continuous state space and continuous action space is a significantly more challenging problem. The `curse of dimensionality' associated with Dynamic Programming (DP)  makes solving such problems computationally intractable, in general. 

The last several years have seen significant progress in deep neural networks based reinforcement learning approaches for controlling unknown dynamical systems, with applications in many areas like playing games \cite{silver2016mastering}, locomotion \cite{lillicrap2015continuous} and robotic hand manipulation \cite{levine2016end}. A number of new algorithms that show promising performance have been proposed \cite{acktr,trpo,ppo} and various improvements and innovations have been continuously developed. However, despite excellent performance on many tasks, reinforcement learning (RL) is still considered very data intensive. The training time for such algorithms is typically really large. Moreover, the techniques suffer from high variance and reproducibility issues \cite{henderson2018deep}. While there have been some attempts to improve the efficiency \cite{gu2016q}, a systematic approach is still lacking. \textit{The issues with RL can be attributed to the typically complex parametrization of the global feedback policy, and the related fundamental question of what this feedback parametrization ought to be?}

\begin{figure}[!ht]
\begin{multicols}{2}
      \subfloat[6-link Swimmer]{\includegraphics[width=0.996\linewidth]{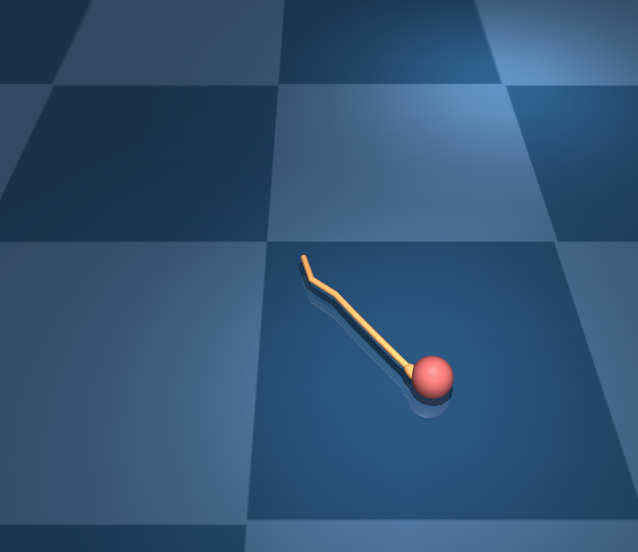}}  
      \subfloat[Fish ]{\includegraphics[width=1\linewidth]{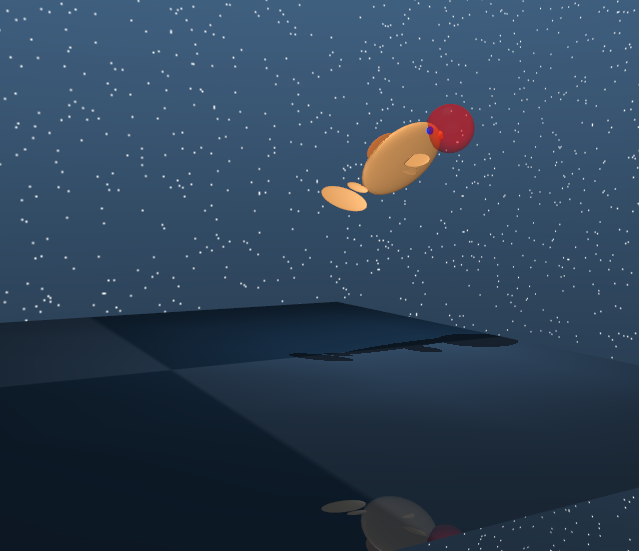}}
\end{multicols}
\begin{multicols}{2}
      \subfloat[Initial state]{\includegraphics[width=1\linewidth]{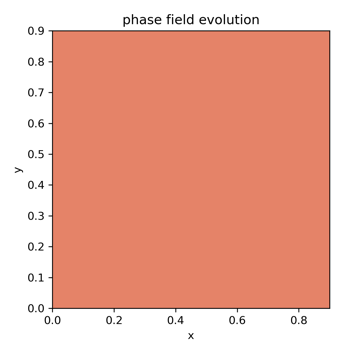}}  
      \subfloat[Final state]{\raisebox{-0.1cm}{\includegraphics[width=1.05\linewidth]{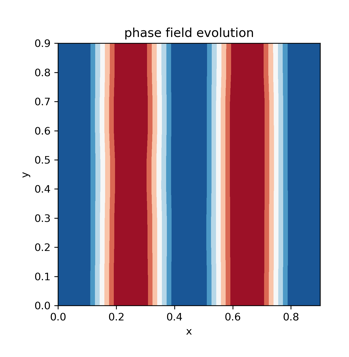}}}
\end{multicols}
\caption{\small Models controlled in this paper including multi-body systems with fluid-structure interactions, and a material microstructure model governed by the Allen-Cahn phase field partial differential equation.}
\label{model_pics}
\end{figure}

\begin{figure}[!hb]
    \centering
\includegraphics[width=1\linewidth]{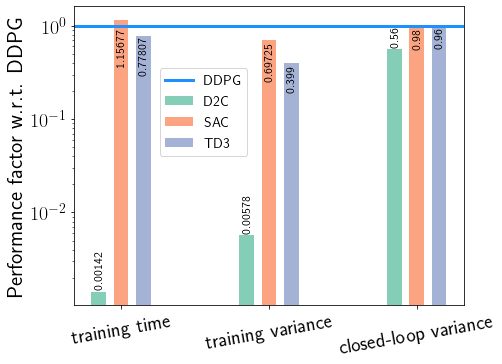}
\caption{\small Local RL approach (D2C) vs global RL approaches. The statistics shown above, found by averaging over all the models simulated, shows that although TD3 and SAC have improvement over DDPG, the local approach is still highly efficient (training time), reliable (training variance), while also having superior closed-loop performance (closed-loop variance), when compared to the global RL approaches.}
\label{outline}
\end{figure}

In this paper, we advocate that for RL to be a) efficient in training, b) reliable in its result, and c) have robust performance to noise, one needs to use a local feedback parametrization as opposed to a global parametrization. Further, this local feedback parametrization consists of an open-loop control sequence combined with a linear feedback law around the nominal open-loop sequence. Searching over this parametrization is highly efficient when compared to the global RL search, can be shown to reliably converge to the global optimum, while having performance that is superior to the global RL solution. In particular, this search is sufficiently fast and reliable that one can recover the optimal global feedback law by replanning whenever necessary. The sole caveat is that these claims are true for a deterministic, albeit unknown, system. However, we show that: 1) theoretically, the deterministic optimal feedback law is near-optimal to fourth order in a small noise parameter to the optimal stochastic law, and 2) empirically, RL methods have difficulty in learning on stochastic systems, so much so that most RL algorithms typically find a feedback law for the deterministic system in which the only noise is an asymptotically vanishing exploration noise.\\

In general, the solution approaches to the problem of controlling unknown dynamical systems can be divided into two broad classes, local and global.

The most computationally efficient among these techniques are ``local" trajectory-based methods such as differential dynamic programming (DDP), \cite{ddp,sddp}, which quadratizes the dynamics and the cost-to-go function around a nominal trajectory, and the iterative linear quadratic regulator (iLQR), \cite{ilqg1,ilqg2}, which only linearizes the dynamics, and thus, is more efficient. Our approach, the Decoupled Data based Control (D2C), requires the model-free/ data-based solution of the open-loop optimization along with a local linear feedback, and thus, we use iLQR with an efficient randomized least squares procedure to estimate the linear system parameters, using simulated rollouts of the system. This local approach to control unknown systems has been explored before \cite{ILQG_tassa2012synthesis, RLHD5}, however, they have always been characterized as ``trajectory optimization" techniques and have been thought of as distinct from RL.  In particular, it was never established that this local approach to RL is highly efficient as well as reliable, in terms of training variance, compared to global approaches. Further, the local approaches are superior in performance in terms of robustness to noise, i.e., they have better ``global" performance compared to global approaches. The reliability of iLQR comes from its guaranteed convergence to the global optimum: we note that iLQR is a Sequential Quadratic Programming (SQP) approach to the optimal control scenario, using which we show that under relatively mild assumptions, iLQR converges to the global minimum from any initial guess. Thus, we can expect that iLQR always yields the same optimal result from different runs. Further, we establish that such local approaches can recover global optimality when coupled with replanning, whenever necessary a la Model Predictive Control (MPC) \cite{Mayne_1,Mayne_2}, which becomes feasible because of the highly efficient and reliable local search that is guaranteed to converge to a globally optimum open-loop solution, and the associated optimal linear feedback law. \\
 

Global methods, more popularly known as approximate dynamic programming   \cite{powell2007approximate,dp_bertsekas} or reinforcement learning (RL) methods \cite{sutton2018reinforcement}, seek to improve the control policy by repeated interactions with the environment while observing the system's responses. The repeated interactions, or learning trials, allow these algorithms to compute the solution of the dynamic programming problem (optimal value/Q-value function or optimal policy) either by constructing a model of the dynamics (model-based) \cite{pilco,Kumar2016, Mitrovic2010}, or directly estimating the control policy (model-free) \cite{sutton2018reinforcement,DDPG,trpo}. Standard RL algorithms are broadly divided into value-based methods, like Q-learning, and policy-based methods, like policy gradient algorithms. Recently, function approximation using deep neural networks has significantly improved the performance of reinforcement learning algorithms, leading to a growing class of literature on `deep reinforcement learning' \cite{acktr,trpo,ppo}. Despite the success, the training time required by such methods, and their variance,  remains prohibitive. 
Our primary contribution in this paper is to show that performing RL via a local feedback parametrization, is highly efficient and reliable, in terms of low variance, when compared to the global approaches. They are also superior in terms of robustness to noise, i.e., their performance is better ``globally" than the global methods.  Finally, we show that global optimality, and hence performance, can be recovered by replanning, that is made feasible via the fast and reliable local planner.\\   

The rest of the paper is organized as follows. In Section \ref{sec2}, the basic problem formulation is outlined. In Section \ref{sec3}, the main decoupling results which solve the stochastic optimal control problem in a 'decoupled open-loop - closed-loop' fashion are briefly summarized. In Section \ref{sec4}, we outline the iLQR based decoupled data based control algorithm and prove its global convergence. In Section \ref{sec5}, we test the proposed approach using typical benchmarking examples with comparisons to a state-of-the-art RL technique. This paper is an extension based on the preliminary results in our conference paper \cite{cdc2021}. In this paper, we provide detailed proof for the near-optimal decoupling principle in Section \ref{sec3}. Also, we show the global convergence of iLQR algorithm and outline the complete D2C algorithm in Section \ref{sec4}. As for empirical results, we add detailed model information used in our simulation experiments and presented new empirical results regarding learning performance of global RL methods when trained on stochastic systems.
\section{Problem Formulation}\label{sec2}
Consider the following discrete time nonlinear stochastic dynamical system:
$
x_{t+1} = F(x_{t}, u_{t},  w_{t}),
$
where $x_t \in \mathbb{R}^{n_x}$,   $u_t \in \mathbb{R}^{n_u}$ are the state  measurement and control vector at time $t$, respectively. The process noise $w_{t}$ is assumed as zero-mean, uncorrelated Gaussian white noise, with covariance $W$. 
The \emph{stochastic optimal control} problem is to find the the control policy $\pi^{o} = \{\pi^{o}_0, \pi^{o}_2, \cdots, \pi^{o}_{T-1} \}$ such that the expected cumulative cost is minimized, i.e., 
$
\pi^{o} = \arg \min_{\pi}  ~ \tilde{J}^{\pi}(x)$, where,  
$\tilde{J}^{\pi}(x) = \mathbb{E}_{\pi} \left[ \sum_{t = 0}^{T-1} c_t(x_{t}, u_{t}) + c_{T}(x_{T}) | x_{0} = x \right]$,
$u_{t} = \pi_{t}(x_{t})$, $c_t(\cdot, \cdot)$ is the instantaneous cost function, and $c_{T}(\cdot)$ is the terminal cost function.  In the following, we assume that the initial state $x_{0}$ is fixed, and denote $\tilde{J}^{\pi}(x_0)$ simply as $\tilde{J}^{\pi}$. 

\section{A Near-Optimal Decoupling Principle}\label{sec3}
We first outline a near-optimal decoupling principle in stochastic optimal control that paves the way for the D2C algorithm described in Section \ref{sec4}. 

Let the dynamics be given by:
\begin{equation}\label{SDE0}
x_t = x_{t-1} + \bar{f}(x_{t-1}) \Delta t + \bar{g}(x_{t-1}) u_t \Delta t + \epsilon \omega_t \sqrt{\Delta t},
\end{equation}
where $\omega_t$ is a white noise sequence, and the sampling time $\Delta t$ is small enough that the $O(\Delta t ^\alpha)$ terms are negligible for $\alpha > 1$. The noise term above stems from Brownian motion, and hence the $\sqrt{\Delta t}$ factor. Further, the incremental cost function $c(x,u)$ is given as:
$c_t(x,u) = \bar{l}_t(x) \Delta t + \frac{1}{2} u'\bar{R}u \Delta t.$ Then, we have the following results. Given sufficient regularity, any feedback policy can then be represented as: $\pi_t (x_t) = \bar{u}_t + K_t^1 \delta x_t + \delta x_t' K_t^2 \delta x_t + \cdots$, where $\bar{u}_t$ is the nominal action with associated nominal state $\bar{x}_t$, i.e., action under zero noise, and $K_t^1, K_t^2, \cdots$ represent the linear and higher order feedback gains acting on the state deviation from the nominal: $\delta x_t = x_t -\bar{x}_t$, due to the noise.\\

\begin{proposition}
\label{prop_e4}
The cost function of the optimal stochastic policy, $J_t$, and the cost function of the ``deterministic policy  applied to the stochastic system", $\varphi_t$, satisfy: $J_t(x) = J_t^0 (x) + \epsilon^2 J_t^1(x) + \epsilon^4 J_t^2 (x)+ \cdots$, and $\varphi_t(x) = \varphi_t^0 (x) + \epsilon^2 \varphi_t^1 (x) + \epsilon^4 \varphi_t^2(x) + \cdots$. Furthermore, $J_t^0(x) = \varphi_t^0 (x)$, and $J_t^1 = \varphi_t^1(x)$, for all $t,x$.
\end{proposition}

\begin{proof}
We show the result for the scalar case for simplicity (and completeness), the vector state case is relatively straightforward to derive, please refer to our paper \cite{arxivpaper1}.
 The DP equation for the given system is given by:
\begin{equation}\label{DP}
J_t(x) = \min_{u_t} \{c_t(x,u_t) + E[J_{t+1}(x')]\},
\end{equation}
where $x' = x + \bar{f}(x)\Delta t + \bar{g}(x) u_t \Delta t + \epsilon \omega_t \sqrt{\Delta t}$ and $J_t(x)$ denotes the cost-to-go of the system given that it is at state $x$ at time $t$. The above equation is marched back in time with terminal condition $J_T(x) = c_T(x)$, and $c_T(\cdot)$ is the terminal cost function. Let $u_t(\cdot)$ denote the corresponding optimal policy. Then, it follows that the optimal control $u_t$ satisfies (since the argument to be minimized is quadratic in $u_t$)
\begin{equation}\label{opt_control}
u_t= -R^{-1}\bar{g}' J_{t+1}^x,
\end{equation}
where $J_{t+1}^x = \frac{\partial J_{t+1}}{\partial x}$.\\
We know that any cost function, and hence, the optimal cost-to-go function can be expanded in terms of $\epsilon$ as:
\begin{equation} \label{f1}
J_t(x) = J_t^0 + \epsilon^2 J_t^1 + \epsilon^4 J_t^2 + \cdots \\
\end{equation}
Thus, substituting the minimizing control in Eq. \ref{opt_control} into the dynamic programming Eq. \ref{DP} implies:
\begin{align}
J_t(x) = \bar{l}_t(x)\Delta t + \frac{1}{2} r(\frac{-\bar{g}}{r})^2 (J_{t+1}^x)^2 \Delta t + J_{t+1}^x \bar{f}(x)\Delta t  \nonumber \\
+ \bar{g}(\frac{-\bar{g}}{r})(J_{t+1}^x)^2 \Delta t+ \frac{\epsilon^2}{2} J_{t+1}^{xx} \Delta t + J_{t+1}(x), \label{f2}
\end{align}
where $J_t^{x}$, and $J_t^{xx}$ denote the first and second derivatives of the cost-to go function. Substituting Eq. \ref{f1} into eq. \ref{f2} we obtain that:
\begin{align}
(J_t^0 + \epsilon^2 J_t^1 + \epsilon^4 J_t^2+\cdots) = \bar{l}_t(x)\Delta t \nonumber \\ + 
\frac{1}{2}\frac{\bar{g}^2}{r}(J_{t+1}^{0,x}
+ \epsilon^2 J_{t+1}^{1,x}+\cdots)^2 \Delta t   \nonumber \\
+(J_{t+1}^{0,x}+ \epsilon^2 J_{t+1}^{1,x}+\cdots) \bar{f}(x) \Delta t \nonumber \\
- \frac{\bar{g}^2}{r} (J_{t+1}^{0,x}+ \epsilon^2 J_{t+1}^{1,x}+\cdots)^2\Delta t  \nonumber \\
+ \frac{\epsilon^2}{2} (J_{t+1}^{0,x}+ \epsilon^2 J_{t+1}^{1,x}+\cdots)\Delta t + J_{t+1}(x). \label{f3}
\end{align}
Now, we equate the $\epsilon^0$, $\epsilon^2$ terms on both sides to obtain perturbation equations for the cost functions $J_t^0, J_t^1, J_t^2 \cdots$. \\
First, let us consider the $\epsilon^0$ term. Utilizing Eq. \ref{f3} above, we obtain:
\begin{align} \label{f4}
J_t^0  = \bar{l}_t \Delta t + \frac{1}{2} \frac{\bar{g}^2}{r}(J_{t+1}^{0,x})^2 \Delta t + 
\underbrace{(\bar{f} + \bar{g}\frac{-\bar{g}}{r} J_t^{0,x})}_{\bar{f}^0} J_t^{0,x}\Delta t + J_{t+1}^0, 
\end{align}
with the terminal condition $J_T^0 = c_T$, and where we have dropped the explicit reference to the argument of the functions $x$ for convenience. \\
Similarly, one obtains by equating the $O(\epsilon^2)$ terms in Eq. \ref{f3} that:
\begin{align}
J_t^1  = \frac{1}{2} \frac{\bar{g}^2}{r} (2J_{t+1}^{0,x}J_{t+1}^{1,x})\Delta t + J_{t+1}^{1,x} \bar{f}\Delta t -\frac{\bar{g}^2}{r} (2J_{t+1}^{0,x}J_{t+1}^{1,x})\Delta t  \nonumber \\
+ \frac{1}{2} J_{t+1}^{0,xx} \Delta t + J_{t+1}^1,
\end{align}
which after regrouping the terms yields:
\begin{align} \label{f5}
J_t^1 = \underbrace{(\bar{f}+ \bar{g} \frac{-\bar{g}}{r} J_{t+1}^{0,x}) }_{= \bar{f}^0}J_{t+1}^{1,x}\Delta t + \frac{1}{2} J_{t+1}^{0,xx} \Delta t + J_{t+1}^1,
\end{align}
with terminal boundary condition $J_T^1 = 0$.
Note the perturbation structure of Eqs. \ref{f4} and \ref{f5}, $J_t^0$ can be solved without knowledge of $J_t^1, J_t^2$ etc, while $J_t^1$ requires knowledge only of $J_t^0$, and so on. In other words, the equations can be solved sequentially rather than simultaneously.\\
Now, let us consider the deterministic policy $u_t^d(\cdot)$ that is a result of solving the deterministic DP equation:
\begin{equation}
\phi_t(x) = \min_{u^d_t} [c_t(x,u^d_t) + \phi_{t+1}(x')],
\end{equation}
where $x' = x + \bar{f}\Delta t + \bar{g} u^d_t \Delta t$, i.e., the deterministic system obtained by setting $\epsilon =0$ in Eq. \ref{SDE0}, and $\phi_t$ represents the optimal cost-to-go of the deterministic system. Analogous to the stochastic case, $u_t^d = \frac{-\bar{g}}{r} \phi_{t+1}^x$.
Next, let $\varphi_t$ denote the cost-to-go of the deterministic policy $u_t^d(\cdot)$ \textit{when applied to the stochastic system, i.e., Eq. \ref{SDE0} with $\epsilon >0$}. Then, the cost-to-go of the deterministic policy, when applied to the stochastic system, satisfies:
\begin{equation}
\varphi_t = c_t(x, u_t^d(x)) + E[\varphi_{t+1}(x')],
\end{equation}
where $x' = x+\bar{f}\Delta t + \bar{g}u_t^d \Delta t + \epsilon\sqrt{\Delta t}\omega_t$. Substituting $u_t^d(\cdot) = \frac{-\bar{g}}{r}\phi_t^x$ into the equation above implies that:
\begin{align}
\varphi_t = \varphi_t^0 + \epsilon^2\varphi_t^1 + \epsilon^4 \varphi_t^2+ \cdots \nonumber\\
= \bar{l}_t \Delta t + \frac{1}{2}\frac{\bar{g}^2}{r}(\phi_{t+1}^x)^2 \Delta t \nonumber
+ (\varphi_{t+1}^{0,x} + \epsilon^2\varphi_{t+1}^{1,x} + \cdots )\bar{f} \Delta t  \nonumber \\
+  \bar{g}\frac{-\bar{g}}{r} \phi_{t+1}^x (\varphi_{t+1}^{0,x} + \epsilon^2\varphi_{t+1}^{1,x} + \cdots ) \Delta t \nonumber\\
+ \frac{\epsilon^2}{2} (\varphi_{t+1}^{0,xx} + \epsilon^2\varphi_{t+1}^{1,xx} + \cdots ) \Delta t 
+  (\varphi_{t+1}^0 + \epsilon^2\varphi_{t+1}^1 + \cdots). \label{f6}
\end{align}
As before, if we gather the terms for $\epsilon^0$, $\epsilon^2$ etc. on both sides of the above equation, we shall get the equations governing $\varphi_t^0, \varphi_t^1$ etc.  First, looking at the $\epsilon^0$ term in Eq. \ref{f6}, we obtain:
\begin{align} \label{f7}
\varphi_t^0 &= \bar{l}_t \Delta t + \frac{1}{2} \frac{\bar{g}^2}{r}(\phi_{t+1}^x)^2 \Delta t  
+ (\bar{f} + \bar{g} \frac{-\bar{g}}{r}\phi_{t+1}^x )\varphi_{t+1}^{0,x}\Delta t \nonumber\\ &+ \varphi_{t+1}^0,    
\end{align}
with the terminal boundary condition $\varphi_T^0 = c_T$. However, the deterministic cost-to-go function also satisfies:
\begin{align}\label{f8}
\phi_t &= \bar{l}_t \Delta t + \frac{1}{2} \frac{\bar{g}^2}{r} (\phi_{t+1}^x)^2 \Delta t  
+( \bar{f}   + \bar{g} \frac{-\bar{g}}{r} \phi_{t+1}^x) \phi_{t+1}^x \Delta t \nonumber\\ &+ \phi_{t+1},
\end{align}
with terminal boundary condition $\phi_T = c_T$. Comparing Eqs. \ref{f7} and \ref{f8}, it follows that $\phi_t = \varphi_t^0$ for all $t$. Further, comparing them to Eq. \ref{f4}, it follows that $\varphi_t^0 = J_t^0$, for all $t$. Also, note that the closed-loop system above, $\bar{f}   + \bar{g} \frac{-\bar{g}}{r} \phi_{t+1}^x = \bar{f}^0$ (see Eq. \ref{f4} and \ref{f5}). \\
Next let us consider the $\epsilon^2$ terms in Eq. \ref{f6}. We obtain:
\begin{align}
\varphi_t^1 = \bar{f} \varphi_{t+1}^{1,x} \Delta t + \bar{g}\frac{-\bar{g}}{r} \phi_{t+1}^x \varphi_{t+1}^{1,x} \Delta t + \nonumber
\frac{1}{2}\varphi_{t+1}^{0,xx} + \varphi_{t+1}^1.
\end{align}
Noting that $\phi_t = \varphi_t^0 $, implies that (after collecting terms):
\begin{align} \label{f9}
\varphi_t^1 = \bar{f}^0 \varphi_{t+1}^{1,x} \Delta t + \frac{1}{2}\varphi_{t+1}^{0,xx} \Delta t + \varphi_{t+1}^1,
\end{align}
with terminal boundary condition $\varphi_T^1 = 0$. Again, comparing Eq. \ref{f9} to Eq. \ref{f5}, and noting that $\varphi_t^0 = J_t^0$,  it follows that $\varphi_t^1 = J_t^1$, for all $t$. This completes the proof of the result. 
\end{proof}

\begin{remark}
The result above shows that the first two terms $J_t^0(x)$ and $J_t^1(x)$ in the perturbation expansion are identical for the optimal deterministic and optimal stochastic policies, when acting on the stochastic system, given they both start at state $x$ at time $t$. This essentially means that the optimal deterministic policy and the optimal stochastic policy agree locally up to order $O(\epsilon^4)$.
\end{remark}
\begin{remark}
It may also be shown that for the optimal deterministic policy, the $\epsilon^0$ term, ${\varphi}^{0}_t$, in the cost, stems from the nominal action ($\bar{u}_t$) of the control policy, the $\epsilon^2$ term, $\varphi^{1}_t$, stems from the linear feedback action of the closed-loop ($K_t^1$), while the higher-order terms stem from the higher-order terms in the feedback law. 
\end{remark}
An important practical consequence of Proposition \ref{prop_e4} is that we can get $O(\epsilon^4)$ near-optimal performance, by wrapping the optimal linear feedback law around the nominal control sequence ($u_t = \bar{u}_t + K_t\delta x_t$), where $\delta x_t$ is the state deviation from the nominal $\bar{x}_t$ state, and replanning the nominal sequence when the deviation is sufficiently large (for convenience, we have dropped the superscript $1$ in denoting the linear feedback gain above). This is similar to the event-driven MPC philosophy of \cite{ETMPC1,ETMPC2}. In the deterministic optimal control problem, the open-loop ($\bar{u}_t$) design is independent of the closed-loop design ($K_t$) which suggests the following ``decoupled" procedure to find the optimal feedback law (locally), i.e., we may first design the optimal open-loop sequence and then find the optimal linear feedback corresponding to said open-loop sequence.

\textbf{\textit{Open-Loop Design.}} First, we design an optimal (open-loop) control sequence $\bar{u}^{*}_t$ for the noiseless system by solving
$
\label{OL}
(\bar{u}^{*}_t)^{T-1}_{t=0} = \arg \min_{(\bar{u}_t)^{T-1}_{t=0}} \sum_{t=0}^{T-1} c_t(\bar{x}_t, \bar{u}_t) + c_T(\bar{x}_T), 
$
~with $\bar{x}_{t+1} = \bar{x}_t+f(\bar{x}_t)\Delta t + g (\bar{x}_{t}) \bar{u}_t\Delta t$. Denote $\mathcal{F}(x) = x + f(x)\Delta t$ and $\mathcal{G}(x) = g(x) \Delta t$ with reference to Eq. \ref{SDE0}.  
The global optimum for this open-loop problem can be found by satisfying the necessary conditions of optimality, as was shown in the first part of \cite{arxivpaper1} using the Method of Characteristics. We propose a data based open-loop solver based on the iLQR algorithm in the next Section.


\textbf{\textit{Closed-Loop Design.}} The optimal linear feedback gain $K_t$ corresponding to the nominal trajectory above is calculated as shown in the following result. In the following, $A_t = \frac{\partial \mathcal{F}}{\partial x}|_{\bar{x}_t} + \frac{\partial \mathcal{G}\bar{u}_t}{\partial x}|_{\bar{x}_t}$, $B_t = \mathcal{G}(\bar{x}_t)$, $L^x_t = \frac{\partial l_t}{\partial x}|_{\bar{x}_t}'$ and $L^{xx}_t = \nabla^2_{xx} l_t |_{\bar{x}_t}$. Let $\phi_t(x_t)$ denote the optimal cost-to-go of the detrministic problem, i.e., Eq. \ref{SDE0} with $\epsilon = 0$. 
\begin{proposition} \label{T-PFC}
Given an optimal nominal trajectory $(\bar{x}_t, \bar{u}_t)$, the backward evolutions of the first and second derivatives, $G_t = \frac{\partial \phi_t}{\partial x}|_{\bar{x}_t}'$ and $P_t = \nabla^2_{xx} \phi_t|_{\bar{x}_t}$, of the optimal cost-to-go function ${\phi}_t({ x_t})$, initiated with the terminal boundary conditions $G_T = \frac{\partial {c}_T}{\partial { x}}\arrowvert_{\bar{x}_T}' $ and $P_T = \nabla^2_{xx} c_T\arrowvert_{\bar{x}_T}$ respectively, are as follows: 
\begin{align}
G_t = L^x_t + G_{t+1}A_t,& \label{OLprop6}\\
P_t = L^{xx}_t + A_t' P_{t+1} A_t - K'_t R_t K_t + G_{t+1} \otimes \tilde R_{t,xx}&  \label{feedback}
\end{align}
for $t = \{0,1,...,T-1\}$, where, 
\begin{align}
    K_t = -R_t^{-1} (B_t'P_{t+1}A_t+(G_{t+1} \otimes \tilde R_{t,xu})'),\label{feedbackgain}
\end{align}
$ \tilde{ R}_{t,xx} = \nabla^{2}_{xx}\mathcal{F}({ x_t})\arrowvert_{{\bar{x}_t}} + \nabla^{2}_{xx}\mathcal{G}({x_t})\arrowvert_{{\bar{x}_t}}{\bar{u}_t}, \\
\tilde{ R}_{t,xu} = \nabla^{2}_{xu}(\mathcal{F}({x_t}) + \mathcal{G}({x_t}){{u}_t}) \arrowvert_{\bar{x}_t, \bar{u}_t}$, where $\nabla^2_{xx}$ represents the Hessian of a vector-valued function w.r.t $x$, similar for $\nabla^2_{xu}$, and $\otimes$ denotes the kronecker product. 
\end{proposition}

\begin{proof}
Consider the Dynamic Programming equation for the deterministic cost-to-go function:
\begin{align*}
{\phi}_t({x_t}) = \mathop{min}_{{u_t}} Q_t({ x_t},{u_t}) = \mathop{min}_{{u_t}}\{c_t({x_t}, {u_t}) + {\phi}_{t+1}({x_{t+1}})\}
\end{align*}
By Taylor's expansion about the nominal state at time $t+1$,
\begin{align*}
{\phi}_{t+1}({x_{t+1}})&={\phi}_{t+1}({ \bar{x}_{t+1}}) + G_{t+1} \delta {x_{t+1}}  \nonumber \\
&~+ \frac{1}{2}\delta {x_{t+1}}' P_{t+1} \delta {x_{t+1}} + q_{t+1}(\delta {x_{t+1}}),
\end{align*}
where $q_{t+1}(\cdot)$ denotes the higher order terms. Substituting the perturbation expansion of the dynamics, \(\delta {x_{t+1}} = A_t \delta { x_t} + B_t \delta {u_t} + r_t(\delta {x_t}, \delta {u_t})\) in the above expansion, where $r_t(\cdot)$ denotes the linearization residual, 
\begin{multline}
 {\phi}_{t+1}({x_{t+1}}) = {\phi}_{t+1}({ \bar{x}_{t+1}}) +  G_{t+1} (A_t \delta { x_t} + B_t \delta { u_t}\\ + r_t(\delta {x_t}, \delta {u_t}) )  + \frac{1}{2}( A_t \delta {x_t} + B_t \delta {u_t} + r_t(\delta {x_t}, \delta {u_t}))' P_{t+1} (A_t \delta {x_t} 
\\+ B_t \delta {u_t}  +  r_t(\delta {x_t}, \delta {u_t}) ) +q_{t+1}(\delta {x_{t+1}}).  
\end{multline}
Similarly, expand the incremental cost at time $t$ about the nominal state, where $s_t(\cdot)$ denotes the higher order terms,
\begin{align*}
c_t({x_t}, {u_t}) &= \bar{l}_t + L^x_t \delta {x_t} + \frac{1}{2} \delta {x_t}' L^{xx}_{t} \delta {x_t} + \frac{1}{2} \delta {u_t}' R_t {\bar{u}_t} \nonumber \\
&+ \frac{1}{2}  {\bar{u}_t}' R_t \delta {u_t}  
+ \frac{1}{2} \delta {u_t}' R_t \delta {u_t}  + \frac{1}{2}{\bar{u}_t}' R_t {\bar{u}_t} + s_t(\delta {x_t}).
\end{align*}
\begin{multline}
Q_t({x_t},{u_t}) = \overbrace{[\bar{l}_t + \frac{1}{2} {\bar{u}_t}' R_t { \bar{u}_t} + {\phi}_{t+1}({\bar{x}_{t+1}}) ]}^{\bar{\phi}_t({\bar{x}_t}, \bar{u}_t)} \nonumber \\
+ L^x_t \delta {x_t} + \frac{1}{2} \delta {x_t}' L^{xx}_{t} \delta {x_t}
+ \delta {u_t}'(B_t' \frac{P_{t+1}}{2} B_t + \frac{1}{2} R_t) \delta {u_t}  \nonumber \\
+ \delta {u_t}'(B_t' \frac{P_{t+1}}{2} A_t \delta {x_t} 
+ \frac{1}{2} R_t {\bar{u}_t} +B_t' \frac{P_{t+1}}{2}r_t)  \nonumber \\
+ (\delta { x_t}' A_t' \frac{P_{t+1}}{2}B_t 
+ \frac{1}{2} {\bar{u}_t} R_t 
+r_t' \frac{P_{t+1}}{2}B_t + G_{t+1}B_t) \delta {u_t} \nonumber \\
+ \delta { x_t}' A_t' \frac{P_{t+1}}{2}A_t \delta {x_t} 
+ \delta {x_t}'A_t' \frac{P_{t+1}}{2} r_t+ (r_t' \frac{P_{t+1}}{2}A_t + G_{t+1} A_t) \delta {x_t}  \nonumber \\
+ r_t' \frac{P_{t+1}}{2}r_t+ G_{t+1}r_t + q_{t+1} +s_t
\equiv \bar{\phi}_t(\bar{x}_t, \bar{u}_t) + H_t(\delta x_t, \delta u_t).
\end{multline}
\begin{flalign*}
\text{Now,~}
\mathop{min}_{{u_t}} Q_t({x_t}, {u_t}) &=  \mathop{min}_{{\bar{u}_t}}  \bar{\phi}_t({\bar{x}_t}, { \bar{u}_t}) + \mathop{min}_{\delta {u_t}} H_t(\delta {x_t} ,\delta {u_t}).
\end{flalign*}
\textbf{First order optimality:} Along the optimal nominal control sequence \(\bar{u}_t\), it follows from the minimum principle that
\begin{align*}
\frac{\partial c_t({x_t}, {u_t})}{\partial { u_t}} + \frac{\partial g({x_t}) }{\partial { u_t}}' \frac{\partial {\phi}_{t+1}({ x_{t+1}})}{\partial {x_{t+1}}}  = 0   
\end{align*}
\begin{equation}
\Rightarrow R_t {\bar{u}_t} + B_{t}' G_{t+1}' = 0.
\end{equation}
By setting \(\frac{\partial H_t(\delta {x_t}, \delta {u_t})}{\partial \delta {u_t}} = 0  \), we get:
\begin{align*}
 \delta {u^{*}_t} &=- (B_t'P_{t+1}B_t+R_t+B_t'P_{t+1} (\tilde R_{t,xu} \otimes \delta {x_t}))^{-1} (R_t  \bar{u}_t \nonumber\\
 &+ B_t' G_{t+1}'+(B_t'P_{t+1}A_t+(G_{t+1} \otimes \tilde R_{t,xu})')\delta {x_t} \nonumber\\
 &+B_t'P_{t+1}r_t+(r_t'P_{t+1}+\delta {x_t}'A_t'P_{t+1}) (\tilde R_{t,xu}\otimes\delta {x_t})).\nonumber
\end{align*}

\noindent By neglecting the $\Delta t^2$ terms,
\begin{align*}
 \delta {u^{*}_t} &= \underbrace{-R_t^{-1} (B_t'P_{t+1}A_t+(G_{t+1} \otimes \tilde R_{t,xu})')}_{K_t}\delta {x_t} \nonumber \\ &\underbrace{-R_t^{-1}(B_t'P_{t+1}r_t
 +(r_t'P_{t+1}+\delta {x_t}'A_t'P_{t+1}) (\tilde R_{t,xu}\otimes\delta {x_t}))}_{p_t} \nonumber \\
 \end{align*}
 \begin{align*}
    \Rightarrow{\delta {u_t^*}= K_t \delta {x_t}+p_t},
 \end{align*}
 where $p_t$ is the second and higher order terms w.r.t. $\delta x_t$. 
Substituting it in the expansion of $\phi_t$ and regrouping the terms based on the order of \(\delta x_t \) (till $2^{nd}$ order), we obtain:
\begin{align*}
{\phi}_t({x_t}) &= \bar{\phi}_t({\bar{x}_t}) + (L^x_t + (R_t {\bar{u}_t} + B_t' G_{t+1}')K_t + G_{t+1}A_t)\delta {x_t}  \nonumber \\
&+ \frac{1}{2}\delta {x_t}' (L^{xx}_{t} + A_t' P_{t+1} A_t - K_t' R_t K_t+G_{t+1} \otimes \tilde R_{t,xx}) \delta {x_t }.
\end{align*}
Expanding the LHS about the optimal nominal state result in the recursive equations in Proposition \ref{T-PFC}.
\end{proof}

\section{Decoupled Data based Control (D2C)}
\label{sec4}
This section presents our decoupled data-based control (D2C) algorithm. We detail the open-loop trajectory design using iLQR, the data-based extension to iLQR and the data-based closed-loop design components of D2C below.
\subsection{Open-Loop Trajectory Design via ILQR}
\label{open_loop_traj_design}
We present an iLQR \cite{ilqg2} based method to solve the open-loop optimization problem. The iLQR iteration which is identical to the Sequential Quadratic Programming (SQP) iteration is shown in Lemma \ref{lm_ilqrsqp} and the algorithm is outlined in Algorithms \ref{model_free_DDP_OL}, \ref{model_free_DDP_OL_FP} and \ref{model_free_DDP_OL_BP}. As for the optimality and convergence of iLQR, We prove that iLQR is guaranteed to converge to the global minimum of the optimal control problem under mild assumptions as shown in Theorem \ref{thm_ilqrglobal} and \ref{thm_ilqrglobalmin}.

\begin{lemma}
\label{lm_ilqrsqp}
Assuming the cost function to minimize is quadratic in control, i.e.:
\begin{align}
\label{nlp}
&\min_{\mathbf{u}} J(\mathbf{x}, \mathbf{u}) = \sum^{T-1}_{t=0}(l_t(x_t)+\frac{1}{2}u_t' Ru_t)+c_T(x_T), \\
&s.t.~~~~x_{t+1} = f(x_t, u_t),~t = 0, 1, \ldots, T-1, \nonumber
\end{align}
where $l_t(\cdot)$ and $c_T(\cdot)$ are the incremental and terminal state cost functions, $\mathbf{x}$ and $\mathbf{u}$ are the state and control trajectories. The iLQR iteration for the above problem is identical to the Sequential Quadratic Programming (SQP) iteration.
\end{lemma}
\begin{proof}
First, we derive the SQP solution. The Lagrangian function of the NLP problem in Eq.~\ref{nlp} can be written as:\\
\begin{align}
\label{eq:true_lgrg}
    \mathcal{L}(\mathbf{x}, \mathbf{u}, \mathbf{\lambda})=J(\mathbf{x}, \mathbf{u})+\sum^{T-1}_{t=0} \lambda_{t+1}' (x_{t+1} - f(x_t, u_t)).
\end{align}
Expanding the cost to second order and the constraint to first order yields the QP problem:
\begin{align}
\label{costest}
    \min_{\delta \mathbf{u}} \delta J(\delta \mathbf{x}, \delta \mathbf{u}) &= \sum^{T-1}_{t=0}[l_{t,x}'\delta x_t+\frac{1}{2}\delta x_t' l_{t,xx}\delta x_t+\bar{u}_t' R \delta u_t \nonumber \\
    &+\frac{1}{2} \delta u_t' R \delta u_t]+c_{T,x}'\delta x_T+\frac{1}{2}\delta x_T' c_{T,xx}\delta x_T,\\
s.t.~~~~\delta x_{t+1} &= f_{x_t}\delta x_t + f_{u_t}\delta u_t,~t = 0, 1, \ldots, T-1, \nonumber
\end{align}
where $\bar{x}_t$ and $\bar{u}_t$ are the resulting state and control of the previous SQP iteration. $l_{t,x} = \frac{\partial l_t}{\partial x}|_{\bar{x}_t}$, $l_{t,xx} = \nabla^2_{xx} l_t|_{\bar{x}_t}$, $c_{T,x} = \frac{\partial c_T}{\partial x}|_{\bar{x}_T}$ and $c_{T,xx} = \nabla^2_{xx} c_T|_{\bar{x}_T}$. The equality constraints are the linearized dynamics where $f_{x_t} = \frac{\partial f}{\partial x}|_{\bar{x}_t}$ and $f_{u_t} = \frac{\partial f}{\partial u}|_{\bar{u}_t}$. Then the Lagrangian function can be estimated with:
\begin{align}
\label{eq:lgrg}
    \mathcal{L}(\delta \mathbf{x}, \delta \mathbf{u}, \mathbf{\lambda}) &= \delta J(\delta \mathbf{x}, \delta \mathbf{u}) + \sum^{T-1}_{t=0}\lambda_{t+1}'(\delta x_{t+1} - f_{x_{t}}\delta x_{t} \nonumber \\
    &- f_{u_{t}}\delta u_{t}).
\end{align}
By satisfying the first order necessary conditions w.r.t. $x_t$, $u_t$ and $\lambda_t$, we have,
\begin{subequations}
\begin{align}
    \lambda_t &=-l_{t,x}-l_{t,xx}\delta x_t+f_{x_t}'\lambda_{t+1} \label{eq:n1}\\
    \delta u_t &=R^{-1}f_{u_t}'\lambda_{t+1}-\bar{u}_t \label{eq:n2}\\
    \delta x_t &= f_{x_{t-1}}\delta x_{t-1}- f_{u_{t-1}}\delta u_{t-1},\label{eq:n3}
\end{align}
\end{subequations}
with boundary condition, $\lambda_T=-c_{T,x}-c_{T,xx}\delta x_T$. Let us now show that the Lagrange multipliers have the form $\lambda_t=-v_t-V_t\delta x_t$ for all $t$. This is trivially true at the terminal timestep where $v_T = c_{T,x}$ and $V_T = c_{T,xx}$. Suppose that $\lambda_{t+1}=-v_{t+1}-V_{t+1}\delta x_t$ for timestep $t+1$, we start with Eq.~\ref{eq:n2} and \ref{eq:n3} at time $ t + 1$ and substitute for $\lambda_{t+1}$:
\begin{align}
    \delta x_{t+1} &= f_{x_t}\delta x_t+ f_{u_t}(R^{-1}f_{u_t}'\lambda_{t+1}-\bar{u}_t) \nonumber \\
    &=f_{x_t}\delta x_t + f_{u_t}(R^{-1}f_{u_t}'(-v_{t+1}-V_{t+1}\delta x_{t+1})-\bar{u}_t).
\end{align}
By collecting terms and solving for $\delta x_{t+1}$ yield:
\begin{align}
    \delta x_{t+1} &= (I+f_{u_t}R^{-1}f_{u_t}'V_{t+1})^{-1}(f_{x_t}\delta x_t \nonumber \\ 
    &-f_{u_t}R^{-1}f_{u_t}'v_{t+1}-f_{u_t}\bar{u}_t).
    \label{eq:x_c}
\end{align}
Substituting $\lambda_{t+1}$ in Eq.~\ref{eq:n1} with the propagation form and $\delta x_{t+1}$ with Eq.~\ref{eq:x_c}:
\begin{align}
    \lambda_t&=-l_{t,x}-l_{t,xx}\delta x_t \nonumber \\
    &+f_{x_t}'(-v_{t+1}-V_{t+1}((I+f_{u_t}R^{-1}f_{u_t}'V_{t+1})^{-1}\nonumber \\
    &\cdot (f_{x_t}\delta x_t-f_{u_t}R^{-1}f_{u_t}'v_{t+1}-f_{u_t}\bar{u}_t))).
    \label{eq:lbd}
\end{align}
After separating zero and first order terms w.r.t. $\delta x_t$, we can show that at timestep $t$, $\lambda_t=-v_t-V_t\delta x_t$, where $v_t$ and $V_t$ can be solved as:
\begin{subequations}
\begin{align}
    v_t &= l_{t,x}+f_{x_t}'v_{t+1}-f_{x_t}'V_{t+1}f_{u_t}(R+f_{u_t}'V_{t+1}f_{u_t})^{-1}\nonumber \\
    &\cdot(f_{u_t}'v_{t+1}+R\bar{u}_t) \\
    V_t &= l_{t,xx}+f_{x_t}'(V_{t+1}^{-1}+f_{u_t}R^{-1}f_{u_t}')^{-1}f_{x_t}\nonumber \\
    &=l_{t,xx}+f_{x_t}'V_{t+1}f_{x_t}-f_{x_t}'V_{t+1}f_{u_t}(R+f_{u_t}'V_{t+1}f_{u_t})^{-1} \label{eq:Vt} \nonumber \\
    &\cdot f_{u_t}'V_{t+1}f_{x_t}.
\end{align}
\end{subequations}
Therefore the Lagrangian multiplyer has the form $\lambda_t=-v_t-V_t\delta x_t$ for all $t$.
Then, by substituting Eq.~\ref{eq:lbd} in Eq.~\ref{eq:n2}, we can solve for $\delta u_t$:
\begin{align}
    \delta u_t
    &=R^{-1}f_{u_t}'(-v_{t+1}-V_{t+1}(f_{x_t}\delta x_t + f_{u_t}\delta u_t))-\bar{u}_t \nonumber \\
    &=-(R+f_{u_t}'V_{t+1}f_{u_t})^{-1}(R\bar{u}_t+f_{u_t}'v_{t+1}\nonumber \\
    &+f_{u_t}'V_{t+1}f_{x_t}\delta x_t).
\end{align}
which can be written in the linear feedback form $\delta u_t =-k_t-K_t\delta x_t$, where $k_t=(R+f_{u_t}'V_{t+1}f_{u_t})^{-1}(R\bar{u}_t+f_{u_t}'v_{t+1})$ and $K_t=(R+f_{u_t}'V_{t+1}f_{u_t})^{-1}f_{u_t}'V_{t+1}f_{x_t}$.\\
Comparing the above results with iLQR \cite{ILQG_tassa2012synthesis}, it is clear that the control update $\delta u_t$ in SQP and iLQR are the same.
In SQP, the update at time $t$ is $(\delta x_t, \delta u_t)$, in which $\delta x_t$ can be recursively solved from Eq.~\ref{eq:n3}, with $\delta x_0=0$ and $\delta u_t =-k_t-K_t\delta x_t$. In the forward pass of iLQR, the state is updated as $\delta x_{t+1} = f(\bar{x}_t+\delta x_t, \bar{u}_t+\delta u_t) - \bar{x}_{t+1}$. Thus the iLQR iterations are always feasible. Suppose the updates are small such that the linearization of the dynamics is valid,
\begin{align}
    \bar{x}_{t+1}+\delta x_{t+1}&=f(\bar{x}_t, \bar{u}_t)+f_{x_t}\delta x_t+f_{u_t}\delta u_t \nonumber \\
    \delta x_{t+1}&=f_{x_t}\delta x_t+f_{u_t}\delta u_t,
    \label{eq:ln_dyn}
\end{align}
which is the same as the SQP iteration in Eq.~\ref{eq:n3}. 
Therefore, it follows that the iLQR iterations are identical to the SQP iterations for the above optimal control problem.
\end{proof}

\noindent To show the convergence of the iLQR algorithm, we list the key assumptions that are needed in the following proof.
\textbf{A1:} The control cost $R$ is chosen to be positive definite for all timesteps, i.e., there exists a constant $\beta_1>0$ such that $\mathbf{d}'R\mathbf{d} \geq \beta_1 \|\mathbf{d}\|^2$ for any $\mathbf{d}$.\\
\textbf{A2:} The cost functions $l_t(\cdot)$ and $c_T(\cdot)$ are chosen such that $\{l_{t,xx}\}$ and $\{c_{T,xx}\}$ are uniformly bounded and positive semi-definite for all t, i.e., there exists a constant $\beta_2 > 0$ such that for each $k$, $\|l_{t,xx}\| \leq \beta_2$, $\|c_{T,xx}\| \leq \beta_2$, $\mathbf{d}'l_{t,xx}\mathbf{d} \geq 0$ and $\mathbf{d}'c_{T,xx}\mathbf{d} \geq 0$ for all $t$ and any $\mathbf{d}$.\\
\textbf{A3:} The starting point and all succeeding iterates lie in some compact set $\mathcal{C}$.\\


\begin{lemma}
\label{lm_descentdir}
Under assumptions \textbf{A1}, \textbf{A2} and \textbf{A3}, if $(\mathbf{x}, \mathbf{u})$ is not a stationary point of the NLP problem in Eq.~\ref{nlp}, the iLQR iteration update $\mathbf{d_s} = [\mathbf{\delta x}'~~\mathbf{\delta u}']'$ is a descent direction for the cost function $J(\mathbf{x}, \mathbf{u})$.
\end{lemma}

\begin{proof}
Denote the constraint from the system dynamics and its gradient as,
\begin{align}
h(\bar{x}_{t+1},\bar{x}_t,\bar{u}_t)&=\bar{x}_{t+1}-f(\bar{x}_t,\bar{u}_t) \nonumber \\ 
\nabla h(\bar{x}_{t+1},\bar{x}_t,\bar{u}_t)&=[I~~- f_{x_t}~~-f_{u_t}].
\end{align}
By satisfying the first order necessary condition,
\begin{align}
    d'_{s,t} \nabla J'_t&=-d'_{s,t} \nabla h'(\bar{x}_{t+1},\bar{x}_t,\bar{u}_t) \lambda_{t+1}-d'_{s,t} B_t d_{s,t},
\end{align}
where $B_t$ is the Hessian of the cost function and $d_{s,t}=[\delta x_{t+1}'~~\delta x_t'~~\delta u_t']'$. The first term on the RHS is zero from Eq.~\ref{eq:ln_dyn}. Using \textbf{A2}, 
\begin{align}
    d'_{s,t} B_t d_{s,t} &= [\delta x'_{t+1}~~\delta x'_t~~\delta u'_t] \begin{bmatrix} l_{t+1,xx}&0&0\\0&l_{t,xx}&0\\0&0&R\end{bmatrix} \begin{bmatrix}\delta x_{t+1}\\ \delta x_t\\ \delta u_t\end{bmatrix} \nonumber \\
    &=\delta x'_{t+1} l_{t+1,xx}\delta x_{t+1}+\delta x'_t l_{t,xx}\delta x_t+\delta u'_t R\delta u_t \nonumber \\
    & > 0,
\end{align}
as $\delta u_t$ can not be zero before reaching a stationary point and $R$ is chosen to be positive definite. Then, it follows that $d'_{s,t} \nabla J'_t < 0$ for all $t$, thus the iterations of iLQR always give a descent direction for the cost function.
\end{proof}

\begin{theorem}
\label{thm_ilqrglobal}
Under assumptions \textbf{A1}, \textbf{A2} and \textbf{A3}, with the line search method in Algorithm \ref{model_free_DDP_OL} and \ref{model_free_DDP_OL_FP}, the iLQR algorithm started at any initial point is guaranteed to converge to a stationary point of the optimization problem in Eq.~\ref{nlp}.
\end{theorem}

\begin{proof}
According to the line search method given in Algorithm \ref{model_free_DDP_OL} and \ref{model_free_DDP_OL_FP}, as well as the fact that all the iterations are feasible, i.e., $h(x_{t+1},x_t,u_t)=0$ for all t, the chosen stepsize $\alpha$ satisfies,
\begin{align}
    \label{eq:linesearch}
    \frac{J(\mathbf{x}^{k+1}, \mathbf{u}^{k+1})-J(\mathbf{x}^k, \mathbf{u}^k)}{\Delta J(\alpha)} \geq \sigma_1 > 0,
\end{align}
where $[\mathbf{x}^{k+1'}, \mathbf{u}^{k+1'}]'=[\mathbf{x}^{k'}, \mathbf{u}^{k'}]'+\alpha \mathbf{d_s^k}$ and $\Delta J(\alpha)$ is defined as $\alpha \nabla J(\mathbf{x}^k, \mathbf{u}^k)' \mathbf{d_s^k}$. In the algorithms below, we write $\Delta J(\alpha)$ as $\Delta cost(\alpha)$. Note that due to our backtracking line search method, $\alpha$ is lower bounded away from zero for all iterations \cite{boggs_tolle_1995}. Substituting for $\Delta J(\alpha)$, Eq.~\ref{eq:linesearch} can be rewritten as, 
\begin{align}
    J(\mathbf{x}^{k+1}, \mathbf{u}^{k+1}) &\leq J(\mathbf{x}^k, \mathbf{u}^k) + \sigma_1 \alpha \nabla J(\mathbf{x}^k, \mathbf{u}^k)' \mathbf{d_s^k} \nonumber\\
    &\leq J(\mathbf{x}^k, \mathbf{u}^k) - \beta_3 |cos\theta| \|\nabla J(\mathbf{x}^k, \mathbf{u}^k)\| \|\mathbf{d_s^k}\|,
\end{align}
where $\beta_3$ is a positive constant, $\theta$ is the angle between $\nabla J(\mathbf{x}^k, \mathbf{u}^k)$ and $\mathbf{d_s^k}$. Since Lemma \ref{lm_descentdir} shows that $\mathbf{d_s^k}$ is always a descent direction, $|cos\theta|$ and $\|\mathbf{d_s^k}\|$ are bounded away from zero. As the cost function at each iteration is finite, we have,
\begin{align}
    \sum_{k=1}^{\infty} \beta_4 \|\nabla J(\mathbf{x}^k, \mathbf{u}^k)\| < \infty.
\end{align}
It follows that,
\begin{align}
    \lim_{k \to \infty} \nabla J(\mathbf{x}^k, \mathbf{u}^k) = 0.
\end{align}
Thus $(\mathbf{x}^k, \mathbf{u}^k)$ from the iLQR algorithm converges to the stationary point of the optimization problem in Eq.~\ref{nlp}, which completes the proof.
\end{proof}

\noindent Next, due to the Method of Characteristics development in \cite{arxivpaper1}, if the dynamics are affine in control and the cost is quadratic in control, it follows that satisfying the necessary conditions for optimality, which iLQR is guaranteed to do as shown above, one is assured that the stationary point is the global minimum of the problem.

\noindent Theorem \ref{thm_ilqrglobal} and the Method of Characterisrics result in paper \cite{arxivpaper1} lead to the following:
\begin{theorem}
\label{thm_ilqrglobalmin}
Under assumptions \textbf{A1}, \textbf{A2} and \textbf{A3}, with a cost function that is quadratic in control and a system whose dynamics are affine in control, the iLQR algorithm is guaranteed to converge to the global minimum of the optimization problem in Eq.~\ref{nlp} from any initial point.
\end{theorem}

\begin{algorithm}
  \caption{\strut Decoupled Data-based Control (D2C) Algorithm using ILQR}
  $\Rightarrow$ \textit{\textbf{Open-loop trajectory optimization}}\\
  {\bf Initialization:} Set state $x = x_0$, initial guess $u^0_{0:T-1}$, line search parameter $\alpha = 1$, regularization $\mu = 10^{-6}$, iteration counter $k = 0$, convergence coefficient $\epsilon = 0.001$, line search threshold $\sigma_1=0.3$.\\
  \While {$cost_{k}/cost_{k-1} < 1 - \epsilon$}{
  \CommentSty{/* backward pass */}\\
  \{$k^{k}_{0:T-1}, K^{k}_{0:T-1}$\} $=$ backward\_pass()\\
  \CommentSty{/* forward pass */}\\
  Initialize $z=1$.\\
      \While{$z < \sigma_1$}{
            Reduce $\alpha$, \\
            $u^{k+1}_{0:T-1}$, $cost_k$, $\Delta cost(\alpha)$ $=$ forward\_pass($u^{k}_{0:T-1},\{k^{k}_{0:T-1}, K^k_{0:T-1}\}$),\\
            $z=(cost_k-cost_{k-1})/\Delta cost(\alpha),$\\
        }
        $k \leftarrow k + 1.$
      }{
      $\bar{u}_{0:T-1} \gets u^{k+1}_{0:T-1}$\\
  }
  $\Rightarrow$ \textit{\textbf{The closed-loop feedback design}}\\
  1. $A_{t}, B_{t} \gets LLS-CD(\bar{u}_{0:T-1},\bar{x}_{0:T-1})$\\
  2. Calculate feedback gain $K_{0:T-1}$ from Eq.~\ref{feedbackgain}.\\
  3. Full closed-loop control policy: $u^*_t = \bar{u}_t + K_t \delta x_t$,\\where $\delta x_t$ is the state deviation from the nominal trajectory.
  \label{model_free_DDP_OL}
\end{algorithm}

\begin{algorithm}
 \caption{\strut Forward Pass}
  {\bf Input:} Previous iteration nominal trajectory - $u^{k}_{0:T-1}$, iLQR gains - $\{k_{0:T-1}, K_{0:T-1}\}$.\\
  Start from $t = 0$, $cost = 0$, $\Delta cost(\alpha)=0$, $\bar{x}_0 = x_0$.\\
  \While {$t < T$}{
    \begin{equation*}
    \begin{split}
        u^{k+1}_{t} &= u_t^{k} + \alpha k_t + K_{t} ({x^{k+1}_t} - {x_t^{k}}),\\
        {x^{k+1}_{t+1}} &= simulate\_forward\_step(x^{k+1}_{t}, {u_t^{k+1}}),\\
        cost &= cost + incremental\_cost({x^{k+1}_{t}}, {u_t^{k+1}}),\\
        \end{split}
    \end{equation*}
    $t \leftarrow t + 1$. 
    }
  $cost = cost + terminal\_cost({x^{k+1}_T}),$\\ \vspace{2pt}
  $\Delta cost(\alpha) = -\alpha \sum_{t=0}^{T-1} k_t' Q_{u_t}-\frac{\alpha^2}{2}\sum_{t=0}^{T-1}k_t'Q_{u_tu_t}k_t.$\\
  {\bf return} $u^{k+1}_{0:T-1}$, $cost$, $\Delta cost(\alpha)$
  \label{model_free_DDP_OL_FP}
\end{algorithm}
\begin{algorithm} 
  \caption{\strut Backward Pass}
  \CommentSty{/* start from the terminal time */}\\
  $t \gets T - 1$.\ \\
  Compute $J_{\bf x_T}$ and $J_{\bf x_T x_T}$ using boundary conditions.\\
  \While {$t >= 0$}{
  \CommentSty{/* Estimate the Jacobians using LLS-CD as shown in Section \ref{sec4}-1): */}\\
  $f_{x_t}, f_{u_t} \gets LLS-CD({ \bar{x}_t},{\bar{u}_t}).$\\
  \CommentSty{/* Obtain the partials of the Q function as follows: */}
    \begin{equation*}
    \begin{split}
    Q_{x_t} &= c_{x_t}  + f_{x_t}' J_{x_{t+1}},\\
    Q_{u_t} &= c_{u_t} + f_{u_t}' J_{x_{t+1}},\\
    Q_{x_t x_t} &= c_{x_t x_t} + f_{x_t}' J_{x_{t+1} x_{t+1}} f_{x_t}, \\
    Q_{u_t x_t} &= c_{u_t x_t} + f_{u_t}' (J_{x_{t+1} x_{t+1}} + \mu I_{n_x \times n_x}) f_{x_t}, \\
    Q_{u_t u_t} &= c_{u_t u_t} + f_{u_t}' (J_{x_{t+1} x_{t+1}}+ \mu I_{n_x \times n_x}) f_{u_t}.\\
    \end{split}
\end{equation*}
  \eIf{$Q_{u_t u_t}$ {\textnormal{is positive-definite}}}{
    \begin{equation*}
    \begin{split}
        k_t &= -Q_{u_t u_t}^{-1} Q_{u_t},\\
        K_t &= -Q_{u_t u_t}^{-1} Q_{u_t x_t}.
    \end{split}
    \end{equation*}
    Decrease~$\mu$.
    }
    {
    Increase $\mu$.\\
    Restart backward pass for current time-step.
    }
  \CommentSty{/* Obtain the partials of the value function $J_t$ as follows: */}
  \begin{equation*}
      \begin{split}
          J_{x_t} &= Q_{x_t} + K_{t}' Q_{u_t u_t} k_t + K_t' Q_{u_t} + Q_{u_t x_t}' k_t,\\
          J_{x_t x_t} &= Q_{x_t x_t} + K_t' Q_{u_t u_t} K_t + K_t' Q_{u_t x_t} + Q_{u_t x_t}' K_t.
      \end{split}
  \end{equation*}
  $t \leftarrow t - 1$ 
  }
  {\bf return} $\{k_{0:T-1}, K_{0:T-1}\}$.
  \label{model_free_DDP_OL_BP}
\end{algorithm}

\subsection{Data-based Extension to ILQR}
\label{sys_id_solve}
ILQR typically requires the availability of analytical system Jacobian, and thus, cannot be directly applied when such analytical gradient information is unavailable (much like Nonlinear Programming software whose efficiency depends on the availability of analytical gradients and Hessians).  In order to make it an (analytical) model-free algorithm, it is sufficient to obtain estimates of the system Jacobians from simulation data, and a sample-efficient randomized way of doing so is described in the following. 

Using Taylor expansion of the non-linear dynamics model in Section \ref{sec2} in the deterministic setting about the nominal trajectory $(\bar{x}_t,  \bar{u}_t)$ on both the positive and the negative sides, we obtain the following central difference equation:
$
    F(\bar{x}_t + \delta x_t, \bar{u}_t + \delta u_t) - F(\bar{x}_t - \delta x_t, \bar{u}_t - \delta  u_t) \\ = 2 \begin{bmatrix} F_{x_t} & F_{u_t} \end{bmatrix} \begin{bmatrix}  \delta {x_t} \\ \delta {u_t} \end{bmatrix} + O(\| \delta {x_t}\|^3 + \| \delta {u_t}\|^3).
$
Multiplying by $\begin{bmatrix} \delta {x_t}' & \delta {u_t}' \end{bmatrix}$ on both sides to the above equation and apply standard Least Square method:
\begin{equation}
    \begin{split}
    &\begin{bmatrix} F_{x_t}~F_{u_t} \end{bmatrix} = M \delta Y_t' (\delta Y_t\delta Y_t')^{-1} \\
    &M = \\
    &\begin{bmatrix} F({\bar{x}_t} + \delta {x_t^{(1)}}, {\bar{u}_t} + \delta {u_t^{(1)}}) - F({ \bar{x}_t} - \delta {x_t^{(1)}}, {\bar{u}_t} - \delta {u_t^{(1)}})\\ F({\bar{x}_t} + \delta {x_t^{(2)}}, {\bar{u}_t} + \delta {u_t^{(2)}}) - F({\bar{x}_t} - \delta {x_t^{(2)}}, {\bar{u}_t} - \delta {u_t^{(2)}}) \\ \vdots  \\F({\bar{x}_t} + \delta {x_t^{(n_s)}}, {\bar{u}_t} + \delta {u_t^{(n_s)}}) - F({\bar{x}_t} - \delta {x_t^{(n_s)}}, {\bar{u}_t} - \delta {u_t^{(n_s)}})\end{bmatrix} \nonumber
     \end{split}
\end{equation}
where `$n_s$' be the number of samples for each of the random variables, $\delta {x_t}$ and $\delta {u_t}$. Denote the random samples as $\delta {X_t} = \begin{bmatrix} \delta {x_t^{(1)}}& \delta {x_t^{(2)}}& \ldots &\delta {x_t^{(n_s)}}\end{bmatrix}$, $\delta {U_t} = \begin{bmatrix} \delta {u_t^{(1)}} &\delta {u_t^{(2)}}& \ldots& \delta {u_t^{(n_s)}}\end{bmatrix}$ and $\delta Y_t = \begin{bmatrix} \delta X_t & \delta U_t \end{bmatrix}$.

We are free to choose the distribution of $\delta {x_t}$ and $\delta {u_t}$. We assume both are i.i.d. Gaussian distributed random variables with zero mean and a standard deviation of $\sigma$.~This ensures that $\delta Y_t \delta Y_t'$ is invertible.

Let us consider the terms in the matrix $\delta Y_t \delta Y_t'=\begin{bmatrix}  \delta {X_t} \delta {X_t}' & \delta {X_t} \delta {U_t}' \\ \delta {U_t} \delta {X_t}'  & \delta {U_t} \delta {U_t}' \end{bmatrix}$.~$\delta {X_t} \delta {X_t}' = \sum_{i=1}^{n_s} \delta {x_t}^{(i)} {\delta {x_t}^{(i)}}'$. Similarly, $\delta {U_t} \delta {U_t}' = \sum_{i=1}^{n_s} \delta {u_t}^{(i)} {\delta {u_t}^{(i)}}'$, $\delta {U_t} \delta {X_t}' = \sum_{i=1}^{n_s} \delta {u_t}^{(i)} {\delta {x_t}^{(i)}}'$ and $\delta {X_t} \delta {U_t}' = \sum_{i=1}^{n_s} \delta {x_t}^{(i)} {\delta {u_t}^{(i)}}'$. From the definition of sample variance, for a large enough $n_s$, we can write the above matrix as 
\begin{equation*}
\begin{split}
    \delta Y_t \delta Y_t' &= \begin{bmatrix} \sum_{i=1}^{n_s} \delta {x_t}^{(i)} {\delta {x_t}^{(i)}}' & \sum_{i=1}^{n_s} \delta {x_t}^{(i)} {\delta {u_t}^{(i)}}' \\ \sum_{i=1}^{n_s} \delta {u_t}^{(i)} {\delta {x_t}^{(i)}}' & \sum_{i=1}^{n_s} \delta {u_t}^{(i)} {\delta {u_t}^{(i)}}'
    \end{bmatrix}\\ &\approx \begin{bmatrix} \sigma^2(n_s - 1) {\text I_{n_x}} & {\text 0_{n_x \times n_u}} \\ 0_{n_u \times n_x} & \sigma^2 (n_s - 1) {\text I_{n_u}}\end{bmatrix} \\
    &= \sigma^2 (n_s - 1){\text I}_{(n_x+n_u) \times (n_x+n_u)}
\end{split}
\end{equation*}
Typically for $n_s \sim O(n_x + n_u)$, the above approximation holds good. The reason is as follows. Note that the above least square procedure converges when the matrix $\delta Y_t \delta Y_t'$ converges to the identity matrix. This is equivalent to estimation of the covariance of the random vector $\delta Y_t = [\delta X_t \ \delta U_t]$ where $\delta X_t$, and $\delta U_t$ are Gaussian i.i.d. samples. Thus, it follows that the number of samples is $O(n_x+n_u)$, given $n_x+n_u$ is large enough (see \cite{HDP-book}). 

This has important ramifications since the overwhelming bulk of the computations in the D2C iLQR implementation consists of the estimation of these system dynamics. Moreover, these calculations are highly parallelizable.
Henceforth, we will refer to this method as `Linear Least Squares by Central Difference (LLS-CD)'. 

\subsection{Data-based Closed-Loop Design}
\label{lqr_design}
The iLQR design in the open-loop part also furnishes a linear feedback law, however, this is not the linear feedback corresponding to the optimal feedback law. In order to accomplish this, we need to use the feedback gain equations (\ref{feedback}). This can be done in a data-based fashion analogous to the LLS-CD procedure above as shown in the Appendix Section \ref{hession_estimation}, but in practice, the converged iLQR feedback gain offers very comparable performance to the optimal feedback gain. The entire algorithm is summarized together in Algorithm \ref{model_free_DDP_OL}. The `forward pass' and `backward pass' algorithms are summarized in Algorithms \ref{model_free_DDP_OL_FP} and \ref{model_free_DDP_OL_BP} respectively.

\section{Empirical Results}
\label{sec5}
This section reports the result of training and performance of D2C on several benchmark examples and its comparison to Deep Deterministic Policy Gradient(DDPG) \cite{plappert2016kerasrl}, Twin-Delayed DDPG(TD3)\cite{td3} and Soft Actor-Critic(SAC)\cite{sac}. DDPG was regarded as an efficient global deep RL method. TD3 and SAC introduced further improvements and outperformed DDPG on many benchmark examples. These three are the current state-of-the-art RL methods, thus are chosen to compare with D2C. The physical models of the system are deployed in the simulation platform `MuJoCo-2.0' \cite{mujoco} as a surrogate to their analytical models. The models are imported from the OpenAI gym \cite{openai} and Deepmind's control suite \cite{DeepRL_bench}. 
In addition, to further illustrate scalability, we test the D2C algorithm on a Material Microstructure Control problem (state dimension of 400) which is governed by a Partial Differential Equation (PDE) called the Allen-Cahn Equation. Please see the supplementary document for more details about the results as well as more experiments.
All simulations are done on a machine with the following specifications: AMD Ryzen 3700X 8-Core CPU@3.59 GHz, with a 16 GB RAM, with no multi-threading.\\

We test the algorithms on four fronts that allow us to test the speed and reliability of the learning, as well as the performance of the learned controllers: 
\begin{enumerate}
\item \textit{Training efficiency}, where we study the time required for training, 
\item \textit{Reliability of the training}, studied using the variance of the resulting answers, 
\item \textit{Robustness of the learned controllers} to differing levels of noise, and hence, a test of the ``global nature" of the synthesized feedback law, and
\item \textit{Learning in stochastic systems}, where we show the effects of a persistent process noise process on learning and performance of the techniques. 
\end{enumerate}
\subsection{Model Description}
\subsubsection{MuJoCo Models}
Here we provide details of the MuJoCo models used in our simulations.\\
\textit{Inverted pendulum} A swing-up task of this 2D system from its downright initial position is considered. \\
\textit{Cart-pole} 
The state of a 4D under-actuated cart-pole comprises the angle of the pole, the cart's horizontal position and their rates. Within a given horizon, the task is to swing up the pole and balance it at the middle of the rail by applying a horizontal force on the cart.\\
\textit{3-link Swimmer} 
The 3-link swimmer model has 5 degrees of freedom and together with their rates, the system is described by 10 state variables. The task is to solve the planning and control problem from a given initial state to the goal position located at the center of the ball. Controls can only be applied in the form of torques to two joints. Hence, it is under-actuated by 3 DOF. \\
\textit{6-link Swimmer} 
The task with a 6-link swimmer model is similar to that defined in the 3-link case. However, with 6 links, it has 8 degrees of freedom and hence, 16 state variables, controlled by 5 joint motors. \\
\textit{Fish} 
The fish model moves in 3D space, the torso is a rigid body with 6 DOF. The system is described by 26 dimensions of states and 6 control channels. Controls are applied in the form of torques to the joints that connect the fins and tails with the torso. The rotation of the torso is described using quaternions. 

\subsubsection{Material Model}
The Material Microstructure is modeled as a 2D grid with periodic boundary, which satisfies the Allen-Cahn equation \cite{ALLEN19791085} at all times. The Allen-Cahn equation is a classical governing partial differential equation (PDE) for phase field models. It has a general form of  
\begin{align}
    &\frac{\partial \phi}{\partial t}= -M(\frac{\partial F}{\partial\phi}-\gamma\nabla^{2}\phi) \label{eq:Dyn_main}
\end{align}    
where $\phi=\phi(x,t)$ is called the ‘order parameter’, which is a spatially varying, non-conserved quantity, and $\nabla^2\phi = \frac{\partial^2 \phi}{\partial x^2} + \frac{\partial^2 \phi}{\partial y^2}$, denotes the Laplacian of a function, and causes a `diffusion' of the phase between neighbouring points. In Controls parlance, $\phi$ is the state of the system, and is infinite dimensional, i.e., a spatio-temporally varying function. It reflects the component proportion of each phase of material system. 
In this study, we adopt the following general form for the energy density function \textit{F}: 
\begin{align}
    &F(\phi;T,h) = \phi^{4}+T\phi^{2}+h\phi \label{eq:D_well} 
\end{align}
Herein, we take both \textit{T}, the temperature, and \textit{h}, an external force field such as an electric field, to be available to control the behavior of the material. In other words, the material dynamics process is controlled from a given initial state to the desired final state by providing the values of \textit{T} and \textit{h}. The control variables $T$ and $h$ are, in general, spatially (over the material domain) as well as temporally varying.

The material model simulated consists of a 2-dimensional grid of dimension 20x20, i.e.,  400 states. The order parameter at each of the grid points can be varied in the range $[-1, 1]$. The model is solved numerically using an explicit, second-order, central-difference-based Finite Difference (FD) scheme. The number of control variables is a fourth of the observation space, i.e., $100$ each for both control inputs $T$ and $h$. Physically, it means that we can vary the $T$ and $h$ values over $2\times 2$ patches of the domain. Thus, the model has 400 state variables and 200 control channels. The control task is to influence the material dynamics to converge to a banded phase distribution as shown in Fig.~\ref{model_pics}(d).

The initial and the desired final state of the model are shown in Fig. \ref{model_pics}(c, d). The model starts at an initial configuration of all states at $\phi=-1$, i.e., the entire material is in one phase. The final state should converge to alternating bands of $\phi=0$ (red) and $\phi=1$ (blue), with each band containing 2 columns of grid points. Thus, this is a very high-dimensional example with a 400-dimensional state and 200 control variables.

\begin{remark}
We note that the methods have access to the same simulation models and no hidden advantage or extra information is provided to either algorithm. Note also that the models, other than the cart-pole and the pendulum examples, lack an analytical description (analytically intractable) and are computational models. Thus, data based methods such as DDPG and D2C can be construed as control synthesis techniques for such analytically intractable models.
\end{remark}
%
\begin{figure*}[!htb]
\centering
\begin{multicols}{3}
    \subfloat[Cartpole]{\includegraphics[width=\linewidth, height=3.3cm]{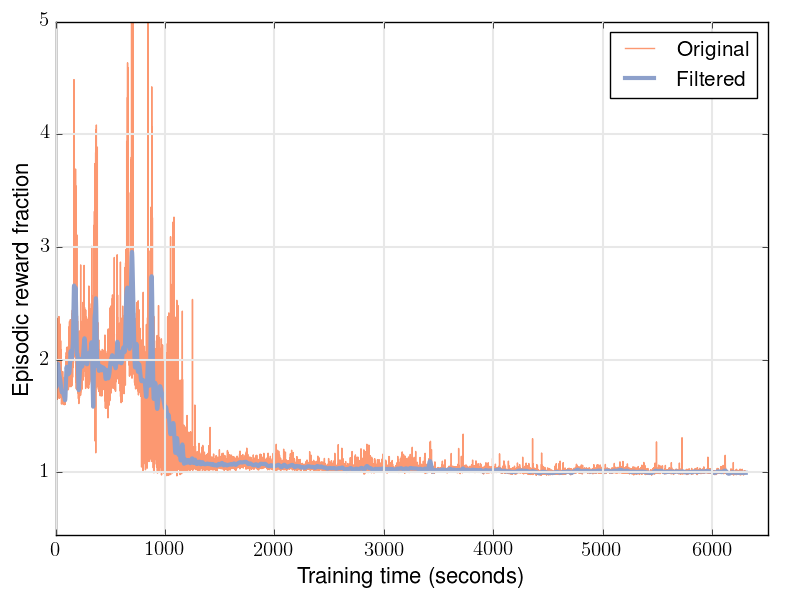}}
    \subfloat[6-link swimmer]{\includegraphics[width=\linewidth, height=3.3cm]{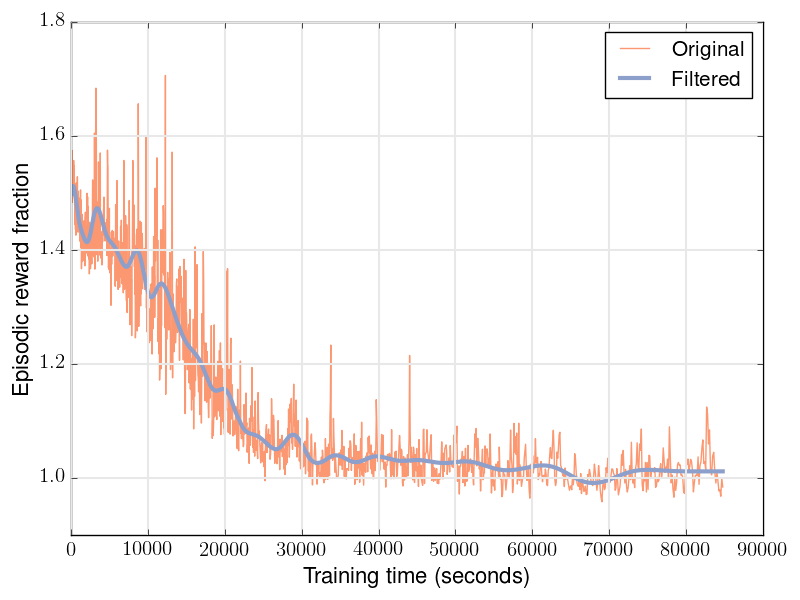}}
    \subfloat[Fish Robot]{\includegraphics[width=1.05\linewidth, height=3.3cm]{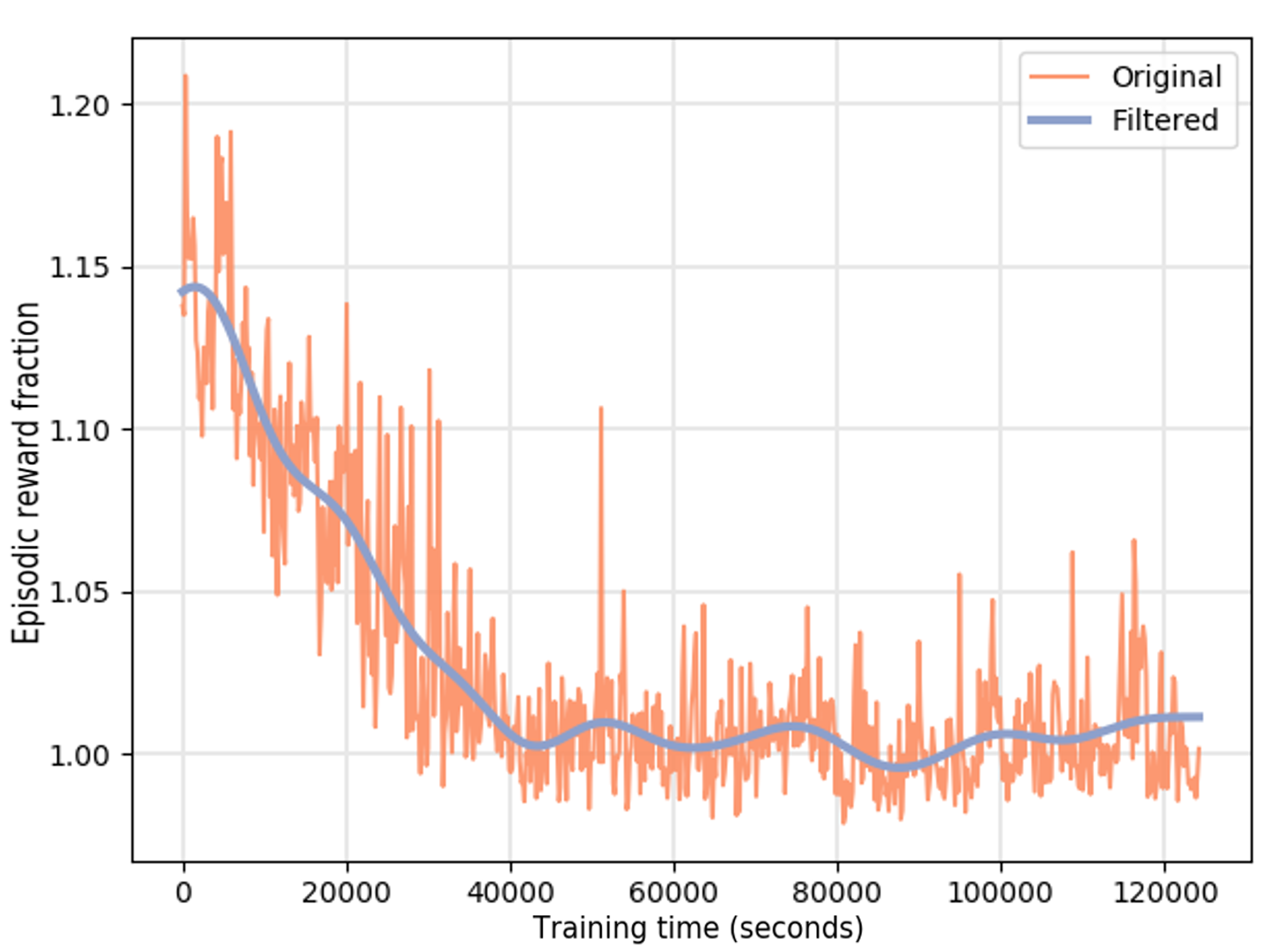}}
\end{multicols}

\begin{multicols}{3}
    \subfloat[Cartpole]{\includegraphics[width=1.05\linewidth, height=3.3cm]{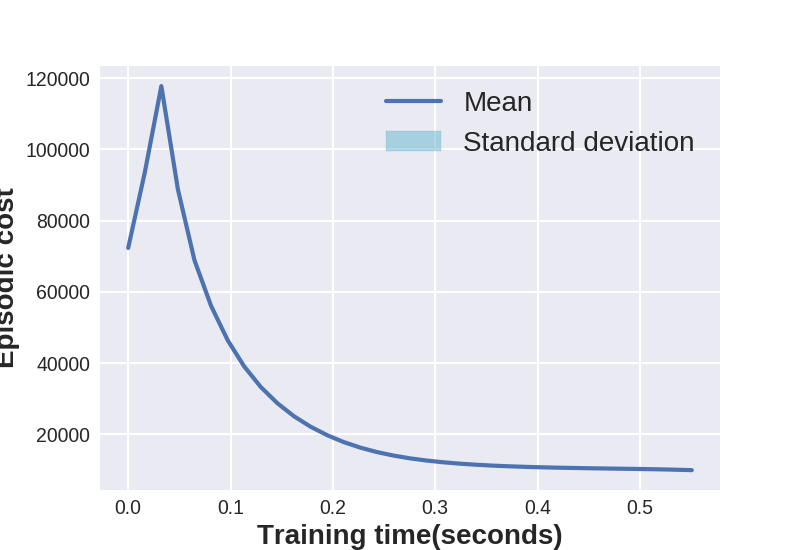}}
    \subfloat[6-link swimmer]{\includegraphics[width=1.05\linewidth, height=3.3cm]{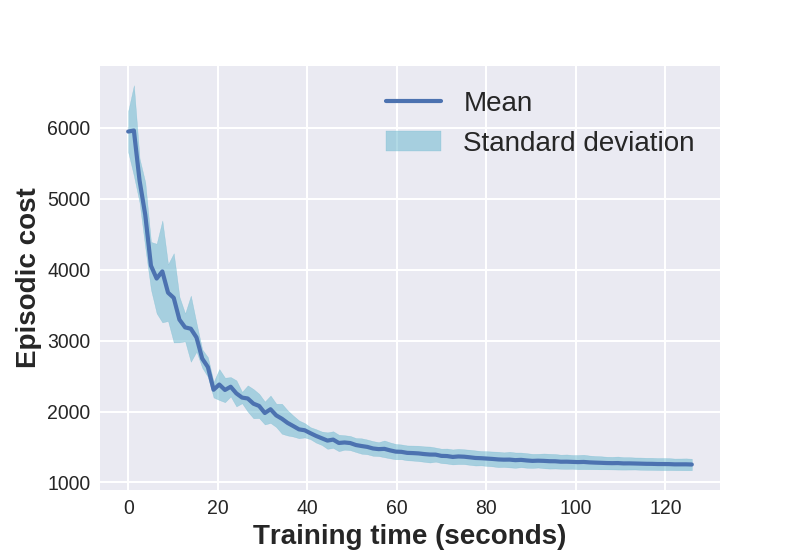}}
    \subfloat[Fish Robot]{\includegraphics[width=1.05\linewidth, height=3.3cm]{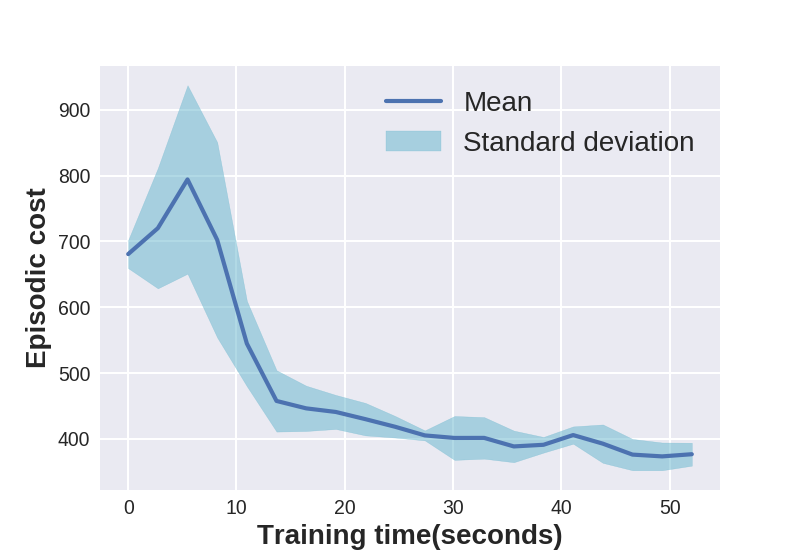}}
\end{multicols}

\begin{multicols}{3}
    \subfloat[Cartpole]{\includegraphics[width=1.\linewidth, height=3.3cm]{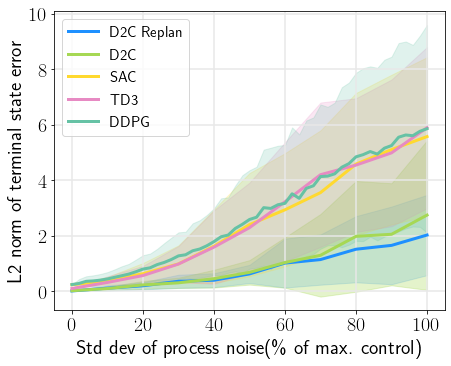}}
    \subfloat[6-link swimmer]{\includegraphics[width=1.05\linewidth, height=3.3cm]{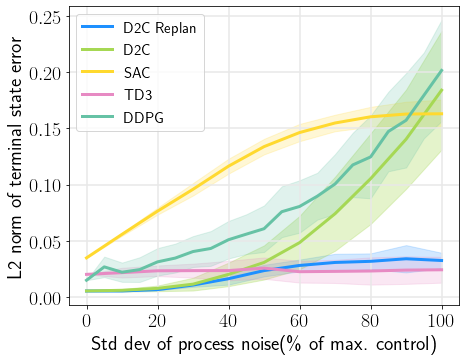}}
    \subfloat[Fish Robot]{\includegraphics[width=1.05\linewidth, height=3.3cm]{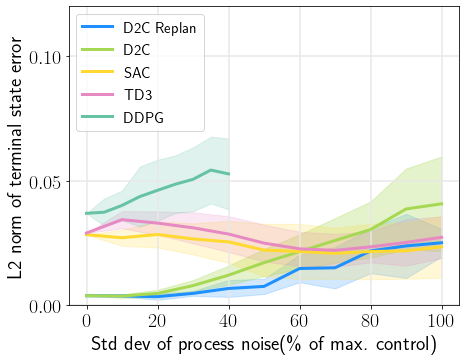}}
\end{multicols}
\caption{\small Top row: Convergence of Episodic cost in DDPG. Middle row: Convergence of Episodic cost D2C. Bottom row: L2-norm of terminal state error during testing in D2C vs RL methods. The solid line in the plots indicates the mean and the shade indicates the standard deviation of the corresponding metric.}
\label{d2c_2_training_testing}
\end{figure*}

\begin{figure}[!htb]
\centering
\subfloat[Material Microstructure]{\includegraphics[width=0.8\linewidth]{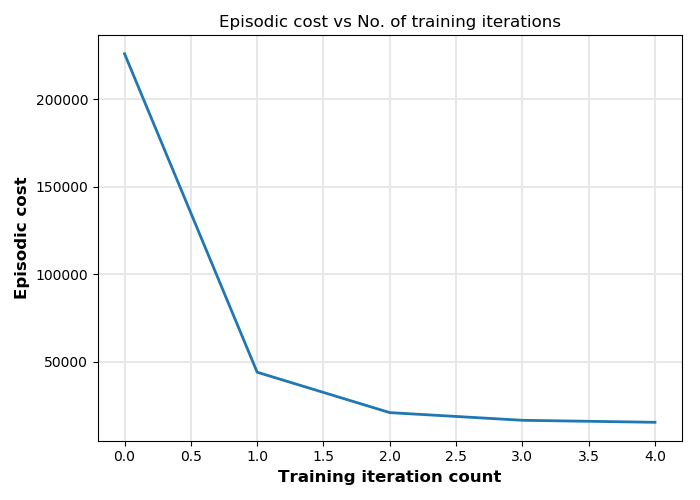}}

\begin{multicols}{4}
\subfloat[Initial]{\includegraphics[width=1.05\linewidth]{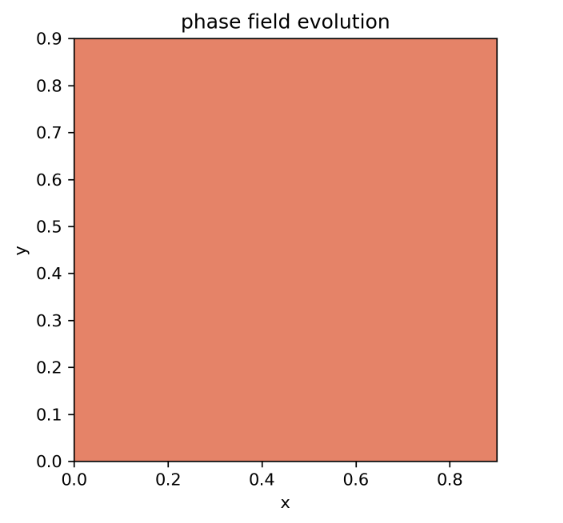}}    
\subfloat[t=0.50s]{\includegraphics[width=0.95\linewidth]{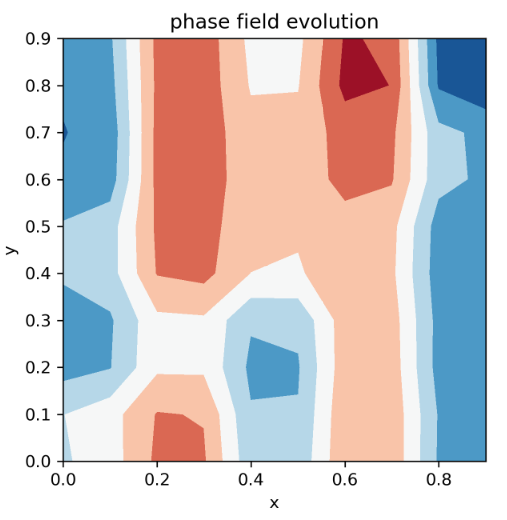}}
\subfloat[t=1.00s]{\includegraphics[width=1.\linewidth]{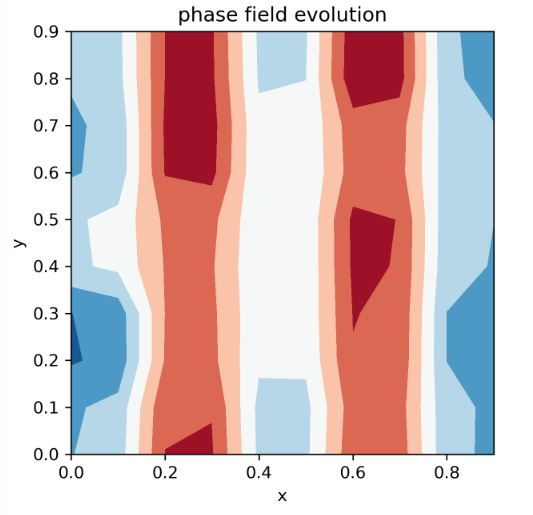}}
\subfloat[t=1.25s]{\includegraphics[width=1.05\linewidth]{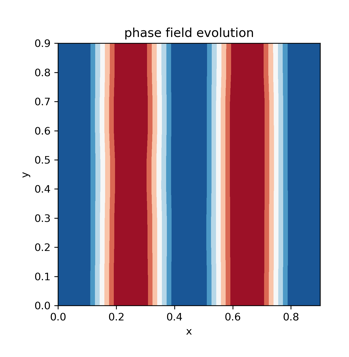}}
\end{multicols}
\begin{multicols}{4}
\subfloat[Initial]{\includegraphics[width=1.1\linewidth]{Figures/Material/Nominal/file0.png}}    
\subfloat[t=0.50s]{\includegraphics[width=\linewidth]{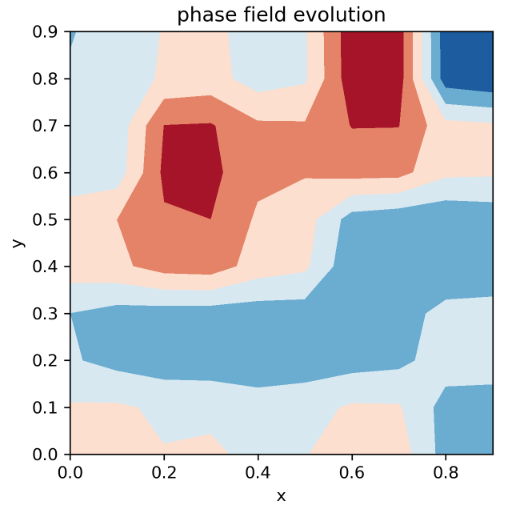}}
\subfloat[t=1.00s]{\includegraphics[width=0.98\linewidth]{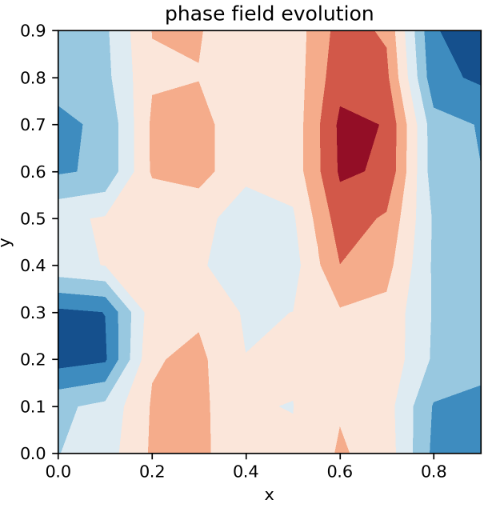}}
\subfloat[t=1.25s]{\includegraphics[width=\linewidth]{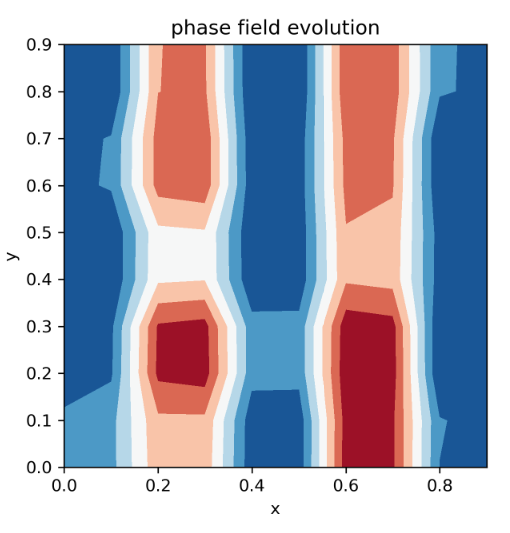}}\\
\end{multicols}
\subfloat[Closed-loop Performance]{\includegraphics[width=0.9\linewidth]{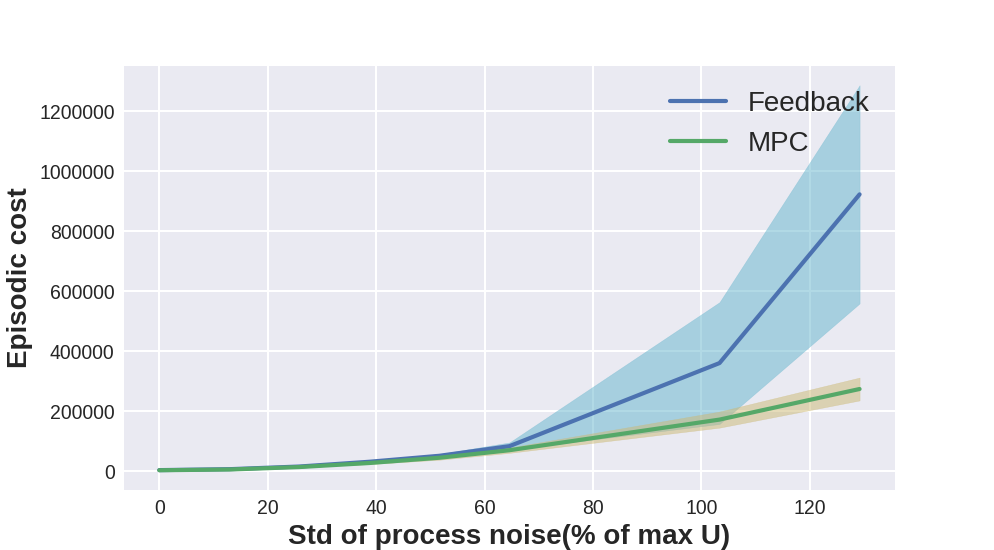}}
\caption{\small Top: Episodic cost vs.~training iteration number in D2C for the Material Microstructure. Middle: Closed-loop trajectories showing the temporal evolution of the spatial microstructure from the initial configuration on the left to the desired configuration on the extreme right. Figs. (b)-(e) No input noise, and (f)-(i) Gaussian input noise at std 50\% $U_{MAX}$. Bottom: Closed-loop performance comparison between D2C with LQR feedback and D2C with replanning.}
\label{material}
\end{figure}

\begin{figure*}[!htb]
\begin{multicols}{2}
      \hspace{50pt}\subfloat[Inverted Pendulum - varied process noise\\ in control ]{\includegraphics[width=0.8\linewidth]{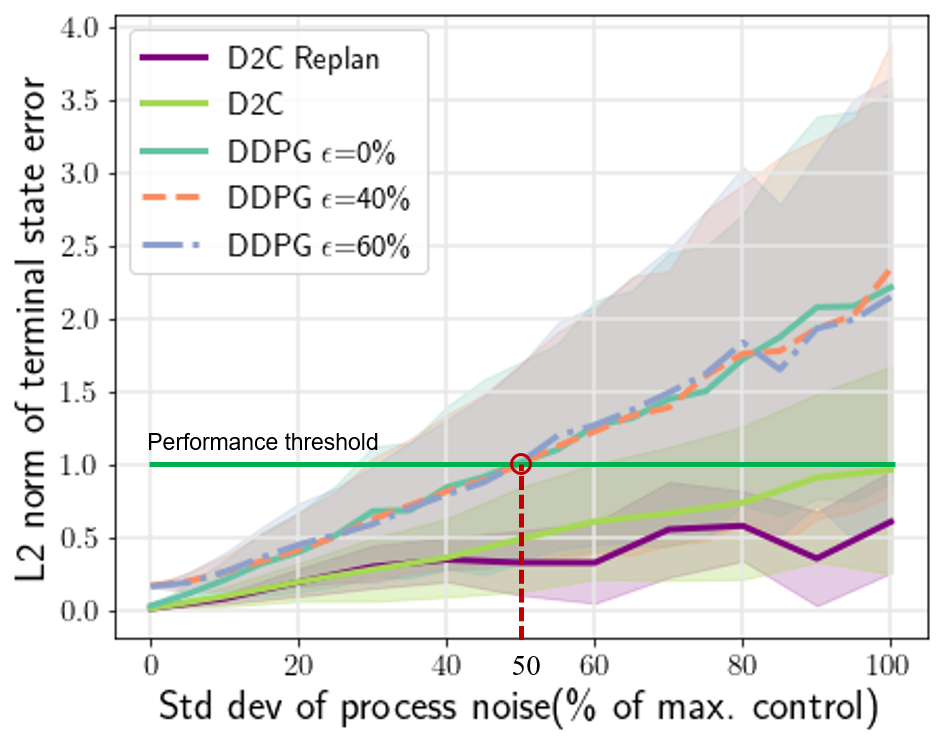}}\subfloat[Inverted Pendulum - varied process noise in state and control ]{\includegraphics[width=0.8\linewidth]{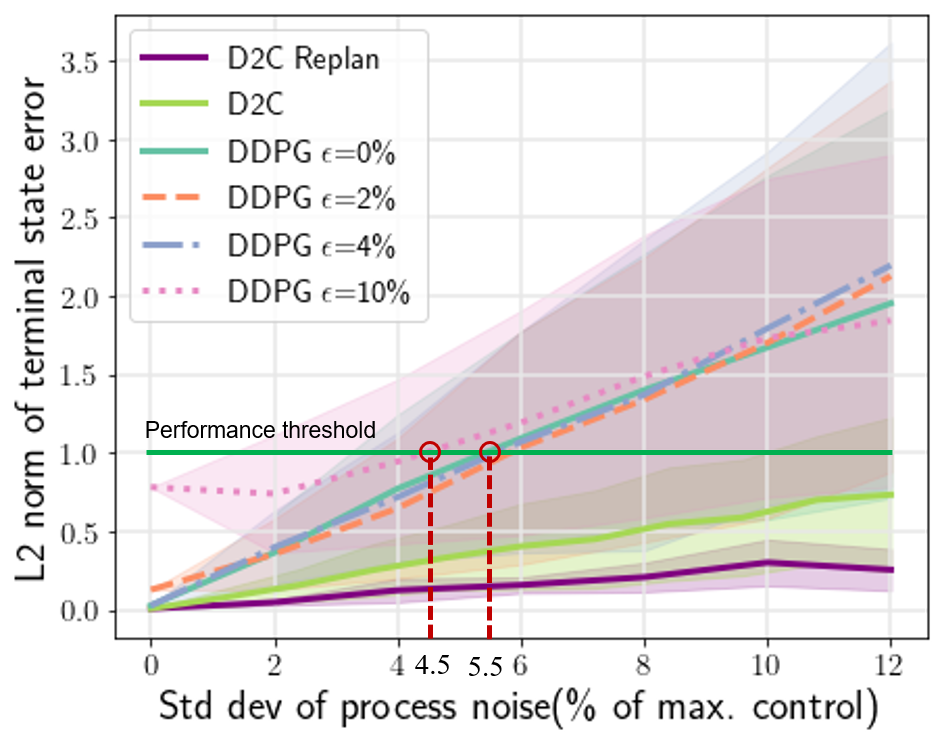}}
\end{multicols}

\begin{multicols}{2}
      \hspace{50pt}\subfloat[Cart-Pole - varied process noise in control ]{\includegraphics[width=0.8\linewidth]{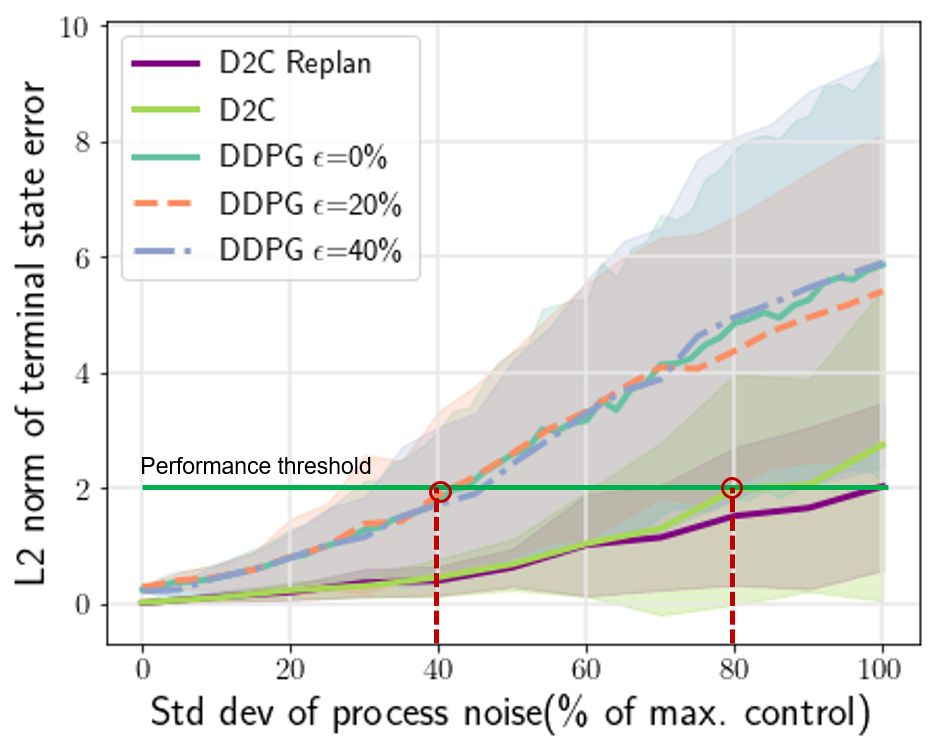}}
     \subfloat[Cart-Pole - varied process noise in state and control ]{\includegraphics[width=0.8\linewidth]{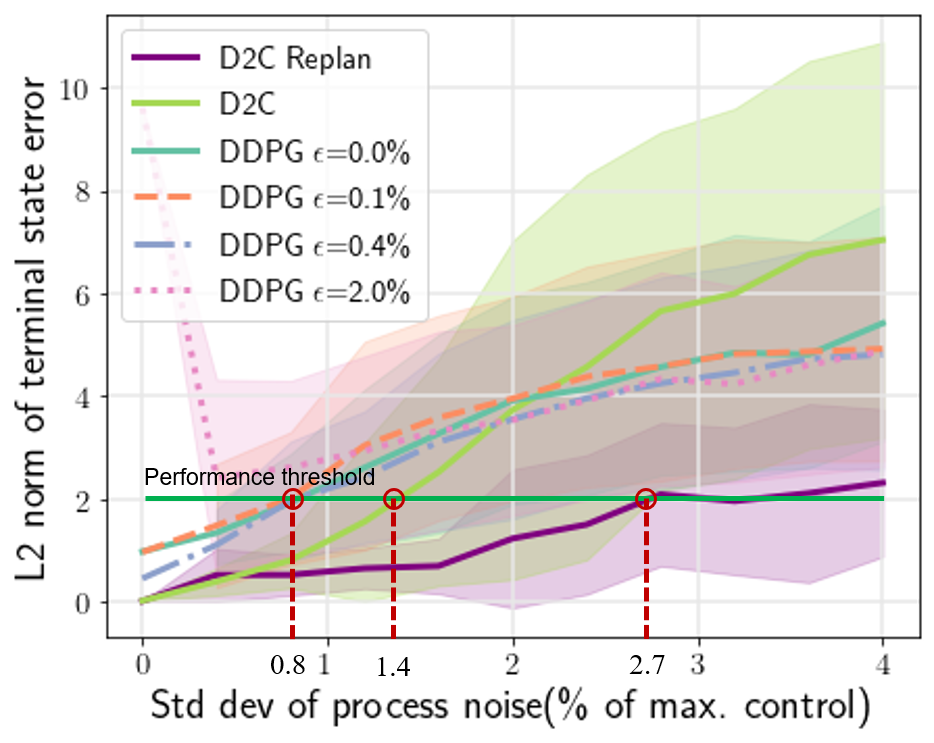}}
\end{multicols}

\caption{L2-norm of terminal state error tested under varying process noise}
\label{pnoise_plots}
\end{figure*}


\begin{figure*}[!htb]
\centering
\begin{multicols}{3}
    \subfloat{\includegraphics[width=1.\linewidth, height=5cm]{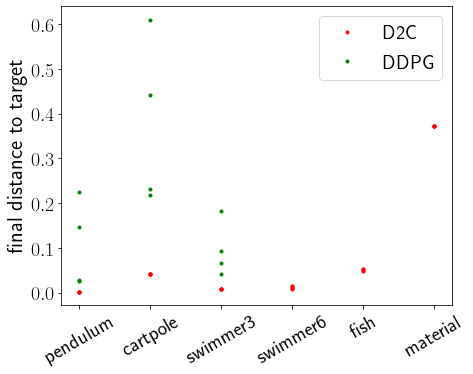}}
    \subfloat{\includegraphics[width=1.05\linewidth, height=5cm]{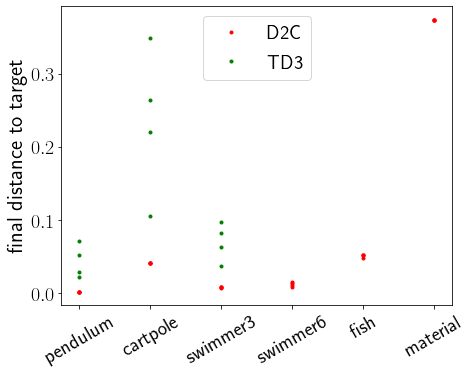}}
    \subfloat{\includegraphics[width=1.\linewidth, height=5cm]{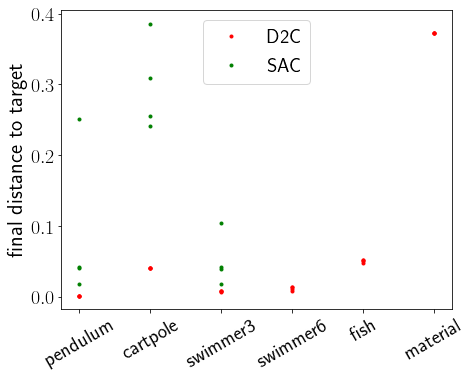}}
\end{multicols}
\caption{Training variance comparison D2C vs. RL methods.}
\label{varcompare}
\end{figure*}
\subsection{Training Efficiency}
We measure training efficiency by comparing the times taken for the episodic cost (or reward) to converge during training. Plots in Figure. 3 show the training process with DDPG and D2C on the systems considered. Table \ref{d2c_comparison_table2} delineates the times taken for training for all four methods respectively. The total time comparison in Table \ref{d2c_comparison_table2} shows that D2C learns the optimal policy orders of magnitude faster than the RL methods. The primary reason for this disparity is the feedback parametrization of the two methods: the RL deep neural nets are complex parametrizations that are difficult to search over, when compared to the highly compact open-loop + linear feedback parametrization of D2C, i.e. the number of parameters optimized during D2C training is the number of actuators times the number of timesteps while the RL parameter size equals the size of the neural networks, which is much larger. Due to the much larger network size, the computation done per rollout is much higher for the RL methods.
From Figure. \ref{material}(a), on the material microstructure problem (a 400-dimensional state and 100-dimensional control), we observe that D2C converges very quickly, even for a very high dimensional system ($d=$ 400), whereas DDPG fails to converge to the correct goal state.\\
We also note the benefit of ILQR here: due to its quadratic convergence properties, the convergence is very fast, when allied with the randomized LLS-CD procedure for Jacobian estimation.\\



\subsection{Closed-loop Performance} 
It might be expected that since the RL methods search over a nonlinear neural network parametrization, it provides a global feedback law while D2C, by design, only provides a local feedback, and thus, the performance of the RL methods might be better globally. To test this hypothesis, we apply noise to the system via the $\epsilon$ parameter, and find the average performance of all the methods at each noise level. This has the effect of perturbing the state from its nominal path, and thus, can be used to test the efficacy of the controllers far from a nominal path, i.e., their global behavior. However, it can be seen from Figure.~\ref{d2c_2_training_testing} bottom row that the performance of D2C is actually better than the RL methods at all noise levels, in the sense that the terminal error of the D2C controller is lower than that of the RL controller. This, in turn implies that albeit the RL methods are theoretically global in nature, in practice, it is reliable only locally, and moreover, its performance is inferior to the local D2C approach. We also report the effect of replanning on the D2C scheme, and it can be seen from these plots that the performance of the replanned controller is far better than both D2C and the RL methods, thereby regaining globally optimal behavior. We also note that the performance of D2C is similar in the high dimensional material microstructure control problem while DDPG fails to converge in this problem (Fig. \ref{material}).

\textit{Replanning with D2C}. Under large noise levels, the local feedback policy found by D2C may not give a good closed-loop performance, and thus we introduce a replanning procedure which re-solves the open-loop design from the current state of the system and wrap another local feedback policy along the new optimal trajectory. During the replanning, we take the current nominal policy as the initial guess. With this warm start, the time and iterations taken in each replanning step are less than solving the open-loop optimization with a ``zero" initial guess in D2C. Under $100\%~U_{max}$ noise, the fish needs 25 seconds and 13 iterations, the 6-link swimmer needs 90 seconds and 51 iterations, on average, for each replanning. As the cart-pole fails under high noise level, it is tested with $40\%~U_{max}$ noise and needs 12 iterations, and 0.5 seconds, on average. Thus, by replanning, the closed-loop performance can be improved with affordable training time increase. 
Finally, we note that the estimation of the feedback gain takes a very small fraction of the training time when compared to the open-loop, even though it is a much bigger parameter: this is a by-product of the decoupling result.

\subsection{Reliability of Training}
For any algorithm that has a training step, it is important that the training result be stable and reproducible, and thus reliable. However, reproducibility is a major challenge that the field of reinforcement learning (RL) is yet to overcome, a manifestation of the extremely high variance of RL training results. 
Thus, we test the training variance of D2C by conducting multiple training sessions with the same hyperparameters but different random seeds. The middle row of Figure.~\ref{d2c_2_training_testing} shows the mean and the standard deviation of the episodic cost data during 16 repeated D2C training runs each. For the cart-pole model, the results of all the training experiments are almost the same. Even for more complex models like the 6-link swimmer and the fish, the training is stable and the variance is small. 
Further investigation into the training results shows that given the set of hyperparameters, D2C always results in the same policy (with a very small variance) unlike the results of the RL methods which have high variance even after convergence, which was reported in \cite{islam2017reproducibility}. We show this in Fig.~\ref{varcompare}, where the final distance to target of the nominal trajectories (i.e., nominal control sequence of D2C and the RL methods) generated from 4 different instances of converged training of all the four methods with identical training hyper-parameters. It can be noted that the D2C results overlap with each other with very small variance while the RL methods' results have a much wider spread. Although TD3 and SAC outperformed DDPG, their training variance is still much higher than D2C. The high variance of the training results, makes it questionable whether the RL methods converge to an optimal solution or whether the seeming convergence is simply the result of the shrinking exploration noise as training progresses. On the other hand, D2C can guarantee the same solution from a converged training. Thus, the advantage of a local approach like D2C in training stability and reproducibility makes it far more reliable for solving data-based optimal control problems when compared to global approaches like DDPG, TD3 and SAC.

\begin{table*}[!htbp]
\hspace{27mm}
\begin{threeparttable}
\begin{center}
\caption{Comparison of the training outcomes of D2C with DDPG, TD3 and SAC.}
\centering
\setlength{\tabcolsep}{2mm}{
\begin{tabular}{|c|c|c|c|c|c|c|c|c|}
\hline
{\bf System}& \multicolumn{4}{|c|}{\bf Training time (in sec.)} & \multicolumn{4}{|c|}{\bf Training variance}\\\cline{2-9}
& {\bf D2C}&{\bf DDPG}&{\bf TD3}&{\bf SAC}& {\bf D2C}&{\bf DDPG}&{\bf TD3}&{\bf SAC}\\
\hline
{Inverted Pendulum} & 0.33 &2261.2&937.6&1462.6&$6.7\times10^{-5}$&0.08&0.02&0.09\\
\hline
{Cart pole}&0.55& 6306.7&4304.5&7407.4&0.0004&0.16&0.09&0.06\\
\hline
{3-link Swimmer}&186.2& 38833.6&48038.0&64035.0&0.0007&0.05&0.02&0.03\\
\hline
{6-link Swimmer} & 127.2& 88160.0&47508.4&57372.7&0.0023&*&*&*\\ 
\hline
{Fish} & 54.8& 124367.6&46238.7&95769.1&0.0016&*&*&*\\
\hline
\end{tabular}
}
\label{d2c_comparison_table2}
\end{center}
    \begin{tablenotes}
      \small
      \item [*] No data because the training time taken is too long.
    \end{tablenotes}
\end{threeparttable}
\end{table*}

\subsection{Learning on Stochastic Systems}\label{sto_sec}
A noteworthy facet of the D2C design is that it is agnostic to the uncertainty, encapsulated by $\epsilon$, and the near-optimality stems from the local optimality (identical nominal control and linear feedback gain) of the deterministic feedback law when applied to the stochastic system. One may then question the fact that the design is not for the true stochastic system, and thus, one may expect RL techniques to perform better since they are applicable to the stochastic system. However, in practice, most RL algorithms only consider the deterministic system, in the sense that the only noise in the training simulations is the exploration noise in the control, and not from a persistent process noise. We now show the effect of adding a persistent process noise with a small to moderate value of $\epsilon$ to the training of DDPG, in the control as well as the state.


We trained the DDPG policy on the pendulum and cart-pole examples. To simulate the stochastic environment, Gaussian i.i.d. random noise is added to all the input channels as process noise. The noise level $\epsilon$ is the noise standard deviation divided by the maximum control value of the open-loop optimal control sequence. Figure.~\ref{pnoise_plots} shows the closed-loop performance of DDPG policies trained and tested under different levels of process noise. The closed-loop performance is measured by the mean and variance of the L2-norm of terminal state error. In the first column, process noise is only added to input channels during training and testing while it is added to both state and input in the second column. The performance threshold is the L2-norm of terminal state error, chosen such that the system at that terminal state is close enough to the target state and it can be stabilized with an LQR controller designed around the target state. For the pendulum, it is chosen such that the angle and angular velocity are smaller than $\ang{40}$ and $\ang{40}/s$ respectively. For the cartpole, the angle and angular velocity of the bar need to be smaller than $\ang{40}$ and $\ang{40}/s$ respectively, the cart position and velocity need to be smaller than $0.8m$ and $1.5m/s$. The thresholds of testing process noise for the policies to keep a decent performance are marked in the plots. When the testing process noise is larger than the threshold, the tested policy can not get close enough to the target state. The problem is greatly exacerbated in the presence of state noise as seen from Figure. \ref{pnoise_plots}(b)(d) that results in bad performance in the different examples for even small levels of noise. Hence, although theoretically, RL algorithms such as DDPG can train on the stochastic system, in practice, the process noise level $\epsilon$ must be limited to a small value for training convergence and/or good policies. Further, the DDPG policy trained at a specific noise level does not perform better than a DDPG policy trained at a different noise level, when tested at that specific noise level. Also, in all the cases shown in these plots, the D2C policy outperforms all the DDPG policies within the performance threshold. Further, the D2C policy with replanning outperforms all the DDPG policies over all tested noise levels. Thus, this begs the question as to whether we should train on the stochastic system rather than appeal to the decoupling result that the deterministic policy is locally identical to the optimal stochastic policy, and thus train on the deterministic system. A theoretical exploration of this topic, in particular, the variance inherent in RL, is the subject of a companion manuscript \cite{arxivpaper3}.

\section{Conclusion}
In this paper, we have presented a data based control method, D2C, for synthesizing the feedback control law for analytically intractable nonlinear optimal control problems. The D2C policy is not global, i.e., it does not claim to be valid over the entire state space, however, seemingly global deep RL methods do not offer better performance as can be seen from our experiments. Further, owing to the fast and reliable open-loop solver, D2C could offer a real time solution even for high dimensional problems when allied with high performance computing. In such cases, one could replan whenever necessary, and this replanning procedure will make the D2C approach global in scope, as we have shown in this paper, albeit not in real time.\\
There might be a sentiment that the comparison with the RL methods is unfair due to the wide chasm in the training times, however, the primary point of our paper is to show theoretically, as well as empirically, that the local parametrization and search procedure advocated via D2C, is a highly efficient and reliable (almost zero variance) alternative that is still superior in terms of closed-loop performance when compared to typical global RL algorithms like DDPG, TD3 and SAC. Thus, for data-based optimal control problems that need efficient training, reliable near-optimal solutions, and robust closed-loop performance, such local RL techniques, coupled with replanning, should be the preferred method over typical global RL methods.\\
Finally, we would like to note that methods such as DDP/ iLQR are termed ``trajectory optimization" techniques and are thought to be distinct from RL algorithms. We have shown in this paper that these methods are indeed (local) RL, or data based control methods, when suitably modified as presented in this paper, and thus, should not be thought of as distinct from RL.



\ifCLASSOPTIONcaptionsoff
  \newpage
\fi
\bibliographystyle{./IEEEtran}
\bibliography{IEEEabrv,refs_D2C,MAP_refs1,MohammadRafi,naveed_references,material_references}

\onecolumn
\appendix
\pagenumbering{arabic}
\setcounter{page}{1}

In this supplementary document, we provide missing details from the empirical results in the paper as well as additional experiments that we did for this work. The outline is as follows: in the first section, we give details and additional results for the training tasks, while in the section following that, we give empirical results for the effect of stochasticity in the dynamics on training. The next section details the implementational of the D2C and DDPG algorithms used in this paper. We close with a section on the system dynamics Hessian estimation using Linear Least Square.

\subsection{Training Comparison: Additional Results}

\emph{Inverted Pendulum.} The training data and performance plots for the 3-link swimmer are shown in Fig.\ref{s3plots}(a), (b) and (c).

\emph{3-link Swimmer.} The training data and performance plots for the 3-link swimmer are shown in Fig.\ref{s3plots}(d), (e) and (f).

\begin{figure*}[!htpb]
\centering
\begin{multicols}{2}
    \subfloat[Inverted Pendulum]{\includegraphics[width=1.05\linewidth, height=6cm]{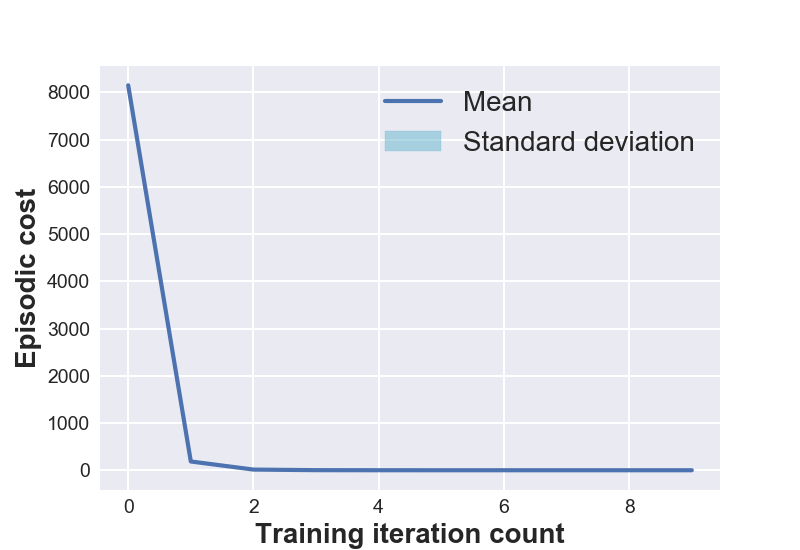}}
    \subfloat[3-link Swimmer]{\includegraphics[width=1.05\linewidth, height=6cm]{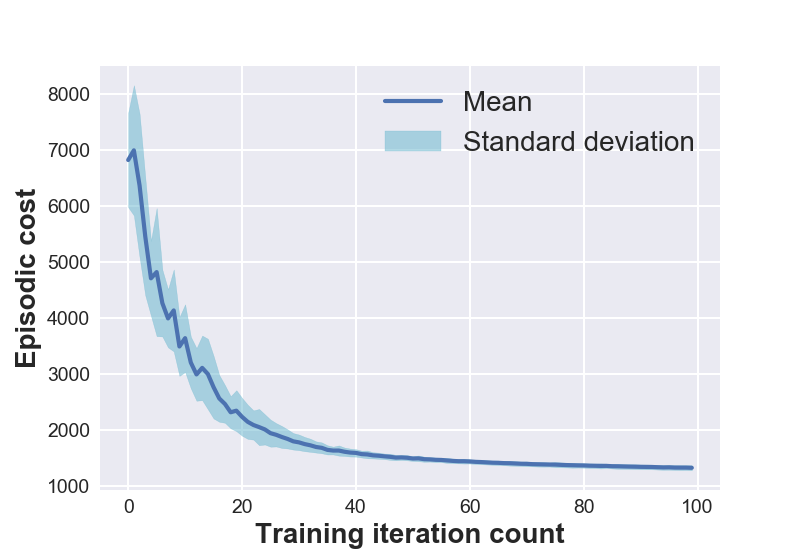}}
\end{multicols}

\begin{multicols}{2}
    \subfloat[Inverted Pendulum]{\includegraphics[width=\linewidth, height=6cm]{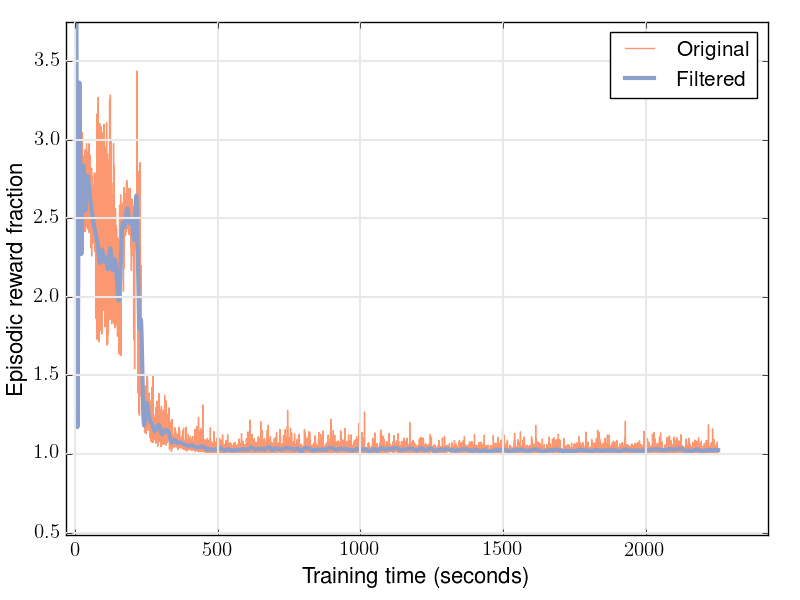}}
    \subfloat[3-link Swimmer]{\includegraphics[width=\linewidth, height=6cm]{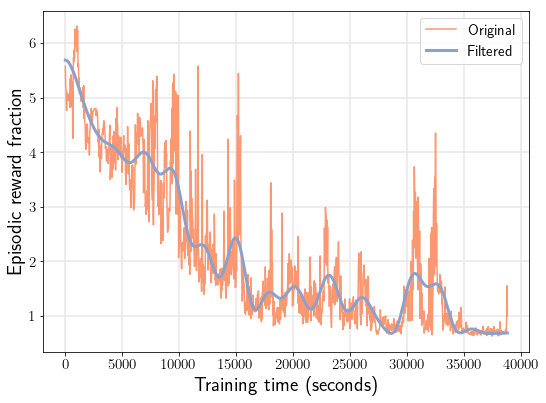}}
\end{multicols}

\begin{multicols}{2}
    \subfloat[Inverted Pendulum]{\includegraphics[width=1.05\linewidth, height=6.5cm]{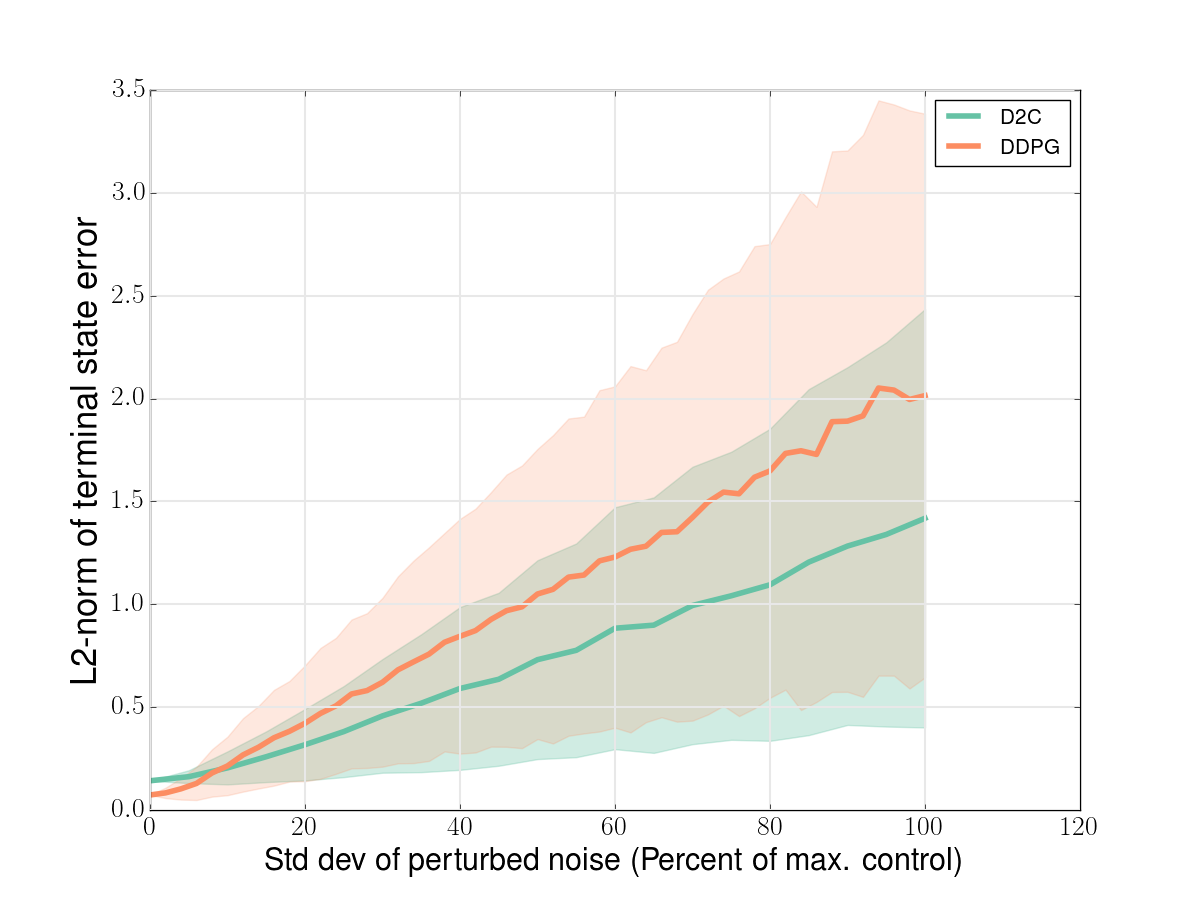}}
    \subfloat[3-link Swimmer]{\includegraphics[width=\linewidth, height=6cm]{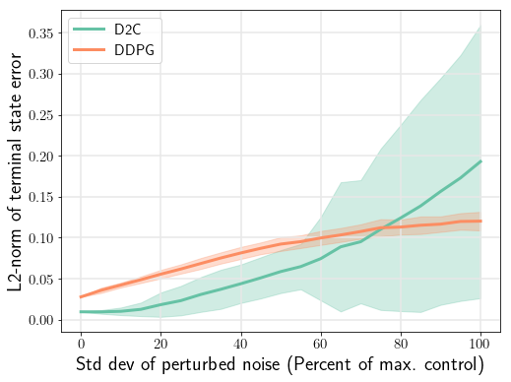}}
\end{multicols}
\caption{Top row: Convergence of Episodic cost in D2C. Middle row: Convergence of Episodic cost in DDPG. Bottom row: L2-norm terminal state error during testing in D2C vs DDPG. The solid line in the plots indicates the mean and the shade indicates the standard deviation of the corresponding metric.}
\label{s3plots}
\end{figure*}

\emph{Data Efficiency, Time Efficiency and Parameter Size.}
In Fig. \ref{d2c_training_rollout}, we give results of training D2C and DDPG with respect to the number of rollouts. This is in addition to the time plot given in Fig. \ref{d2c_2_training_testing}. Note that the time efficiency of D2C is far better than DDPG while the data efficiency of DDPG seems better (in the swimmers and fish), in that it needs fewer rollouts for convergence for the swimmers (albeit it does not converge to a successful policy for the fish, in the time allowed for training). In our opinion, the wide time difference is due to the disparity in the size of the feedback parametrization between the two methods.
Table \ref{parasize} summarizes the parameter size during the training of D2C and DDPG. The number of parameters optimized during D2C training is the number of actuators times the number of timesteps while the DDPG parameter size equals the size of the neural networks, which is much larger. Due to the much larger network size, the computation done per rollout is much higher for DDPG. \textit{We note here that these are the minimal sizes required by the deep nets for convergence and we cannot really make them smaller without loss of convergence}. 
This is not surprising as the D2C primarily derives its efficiency from its compact parametrization of the feedback law. \\
Finally, regarding the seeming sample efficiency of DDPG, it is true that DDPG converges to ``a solution" in fewer rollouts but that does not mean it converges to the optimal solution. Please see the paper: "On the Convergence of Reinforcement Learning", to see the sample complexity required for an ``accurate" solution, which turns out to be double factorial-exponential in the order of the approximation desired. 
\begin{table}[ht]
\caption{Parameter size comparison between D2C and DDPG}
\label{parasize}
\centering
\vspace{0.1in}
\begin{threeparttable}
\setlength{\tabcolsep}{0.6mm}{
\begin{tabular}{|c|c|c|c|c|}
\hline
System& No. of  & No. of & No. of  &No. of \\
&steps&actuators&parameters&parameters\\
&&&optimized&optimized\\
&&& in D2C& in DDPG\\
\hline
Inverted&30  &1&30&244002\\
Pendulum& & &&\\
\hline
Cart pole& 30 &1& 30&245602\\
\hline
3-link & 950&2 & 1900&251103\\
Swimmer& & &&\\
\hline
6-link &900&5 &4500&258006\\
Swimmer &&&& \\ 
\hline
Fish&1200 &6 & 7200&266806\\
\hline
Material &100&800 &80000&601001\\
microstructure &&&& \\ 
\hline
\end{tabular}
}
\vspace{0.1in}
\end{threeparttable}
\end{table}

\begin{figure*}[!htb]
\begin{multicols}{3}
      \subfloat[Inverted Pendulum D2C]{\includegraphics[width=1.05\linewidth]{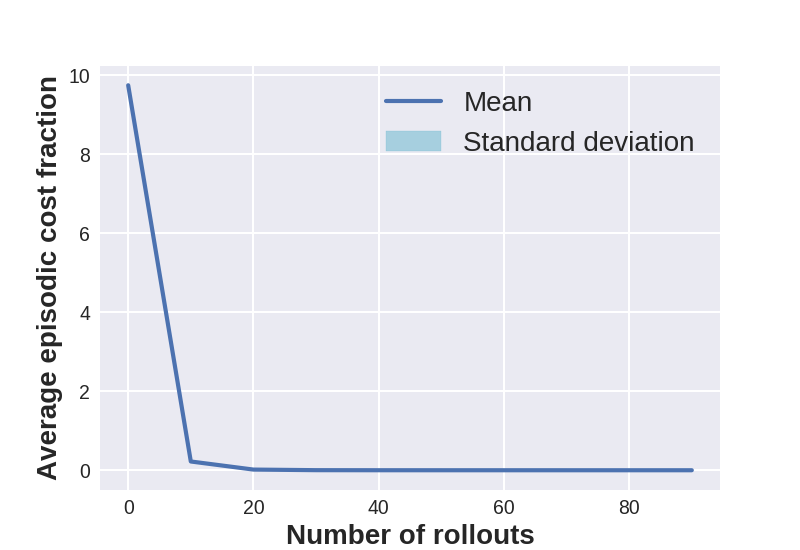}}  
      \subfloat[Cart-Pole D2C ]{\includegraphics[width=1.05\linewidth]{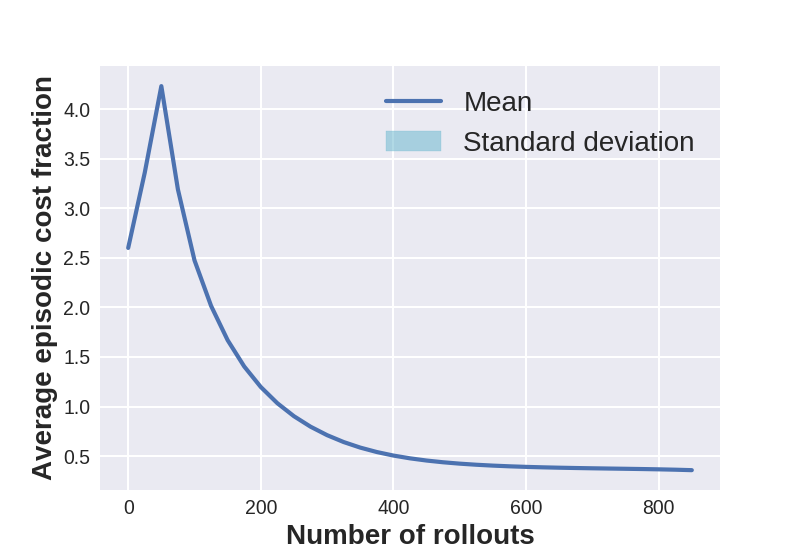}}
      \subfloat[3-link Swimmer D2C ]{\includegraphics[width=1.05\linewidth]{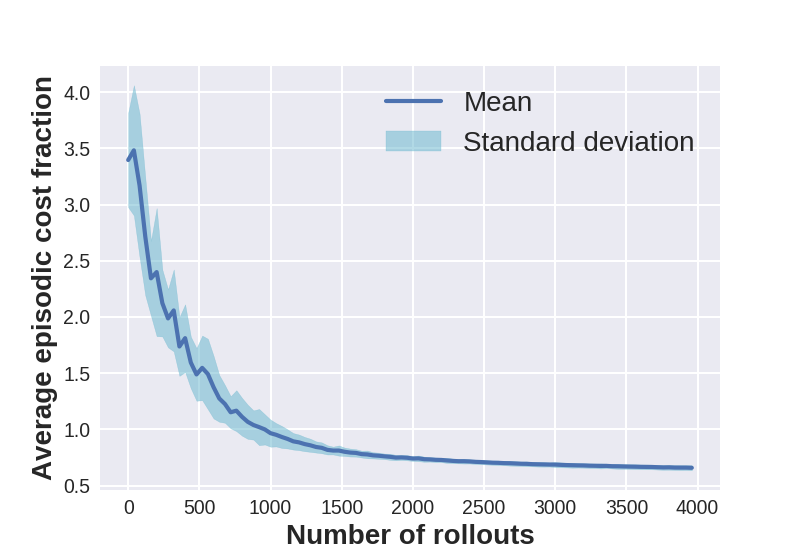}}
\end{multicols}
\begin{multicols}{3}
      \subfloat[Inverted Pendulum DDPG]{\includegraphics[width=1.02\linewidth]{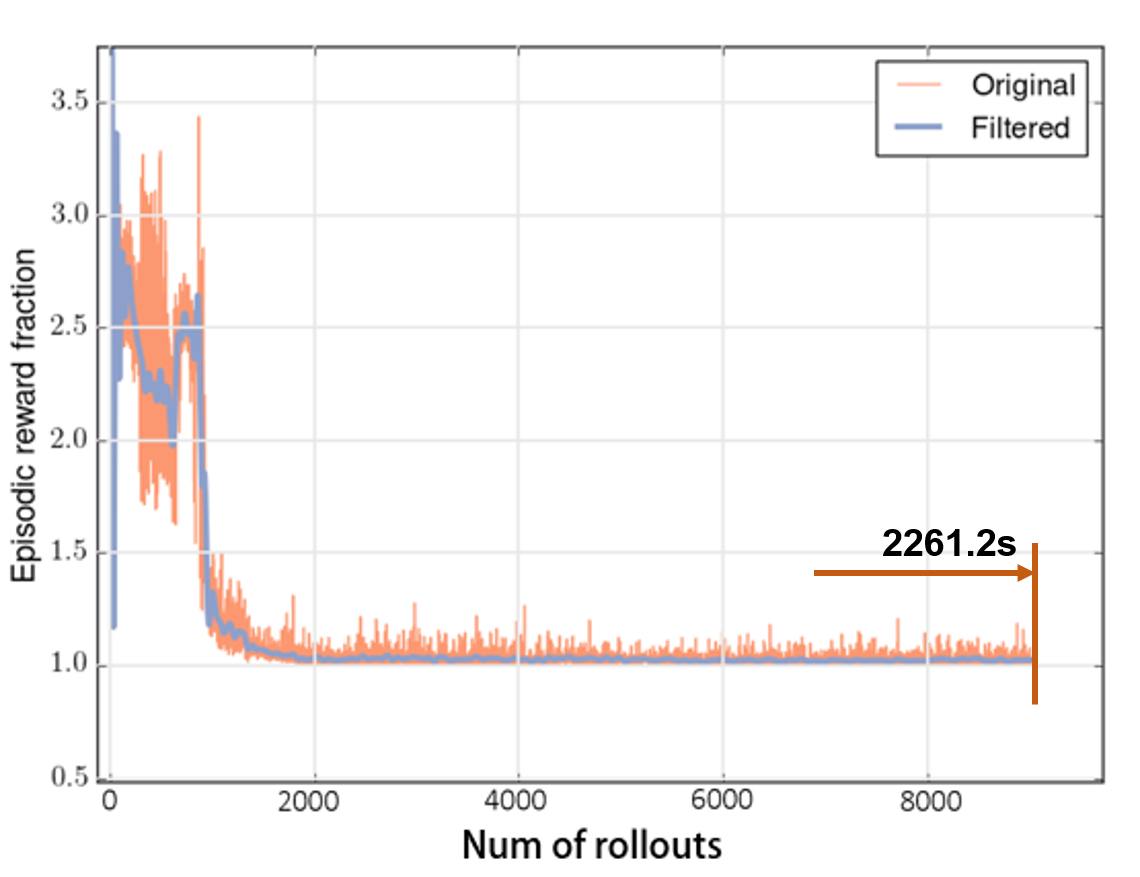}}  
      \subfloat[Cart-Pole DDPG ]{\includegraphics[width=1.02\linewidth]{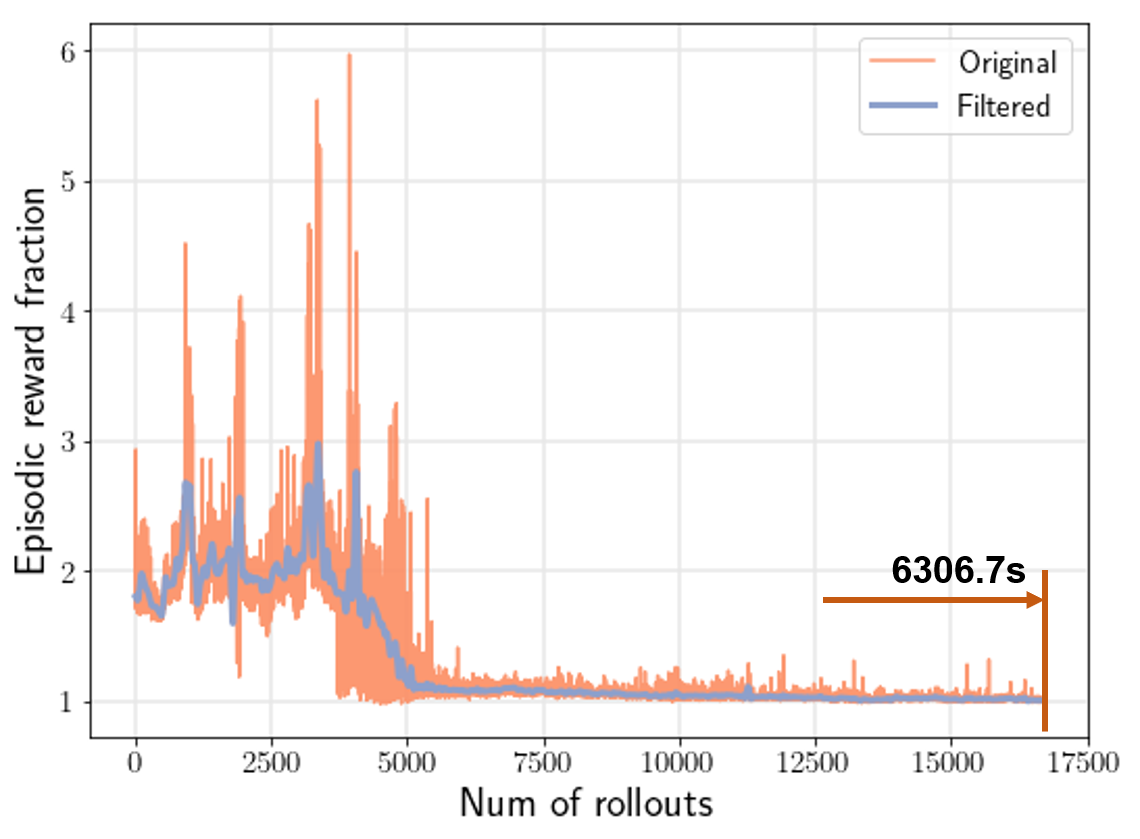}}
      \subfloat[3-link Swimmer DDPG]{\includegraphics[width=1.02\linewidth]{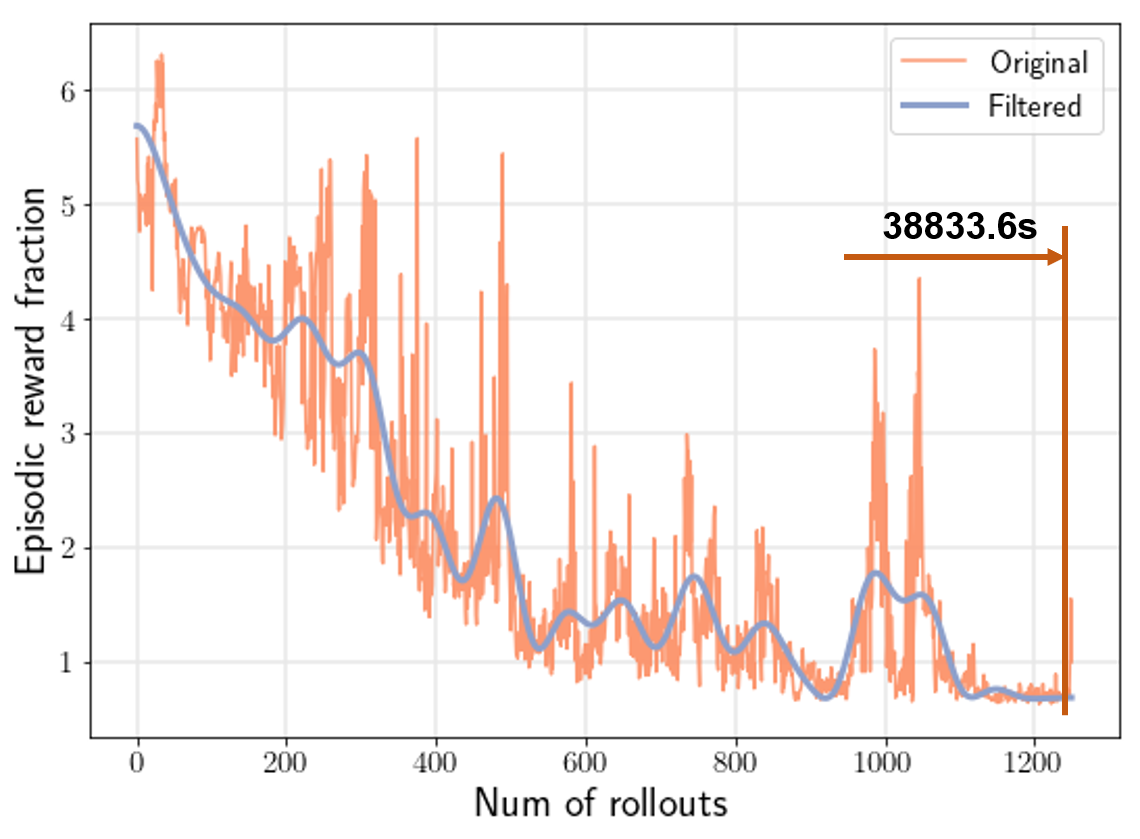}}
\end{multicols}

\begin{multicols}{3}
      \subfloat[6-link Swimmer D2C ]{\includegraphics[width=1.05\linewidth]{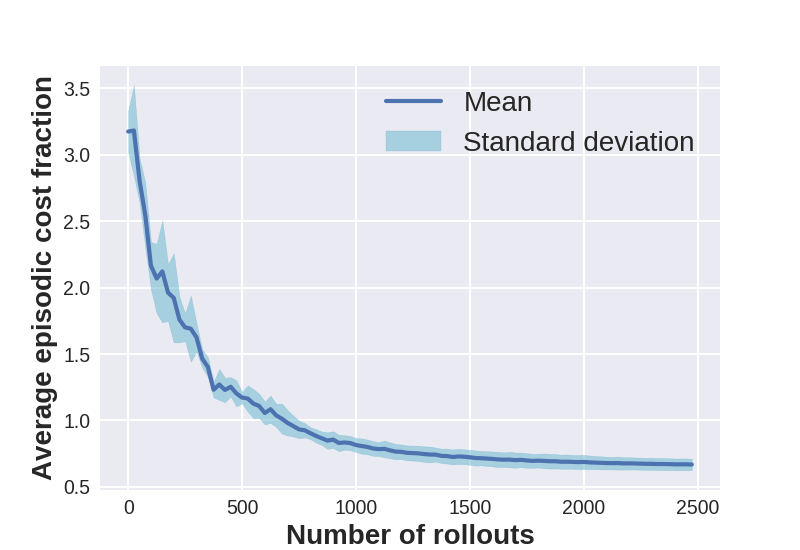}}
      \subfloat[Fish D2C]{\includegraphics[width=1.05\linewidth]{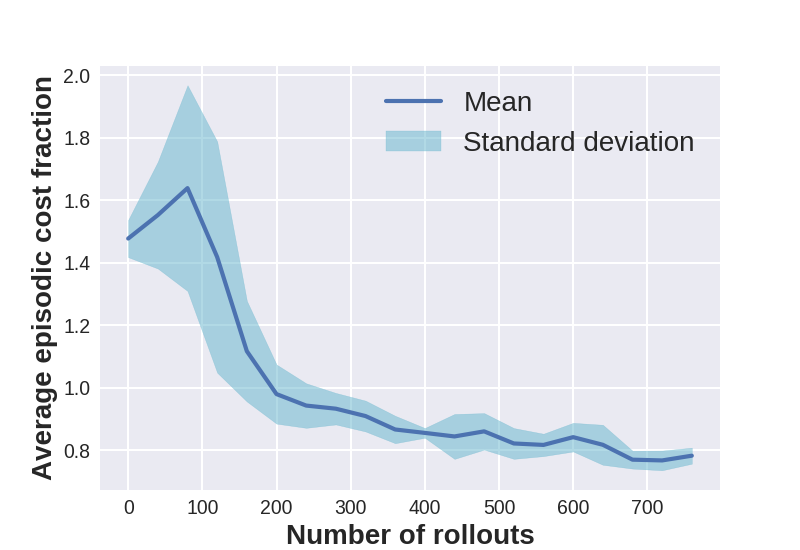}}
      \subfloat[Microstructure D2C]{\includegraphics[width=1.05\linewidth]{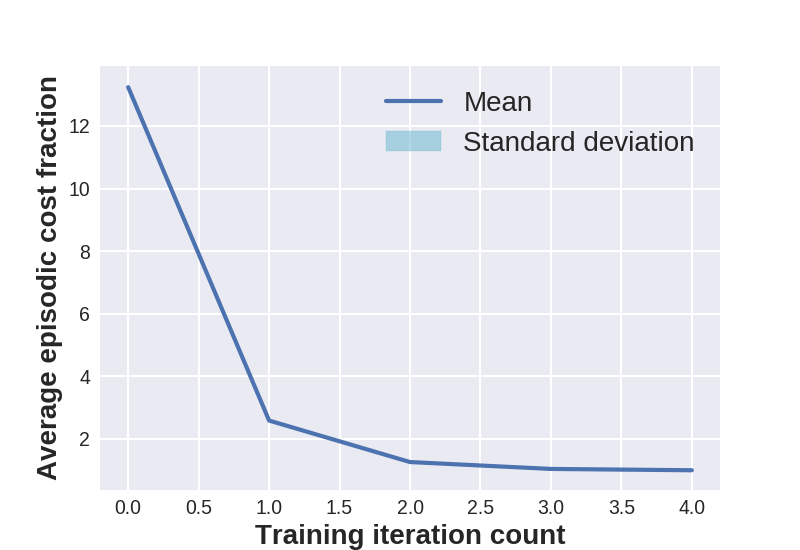}}
\end{multicols}
\begin{multicols}{3}
      \subfloat[6-link Swimmer DDPG ]{\includegraphics[width=1.02\linewidth]{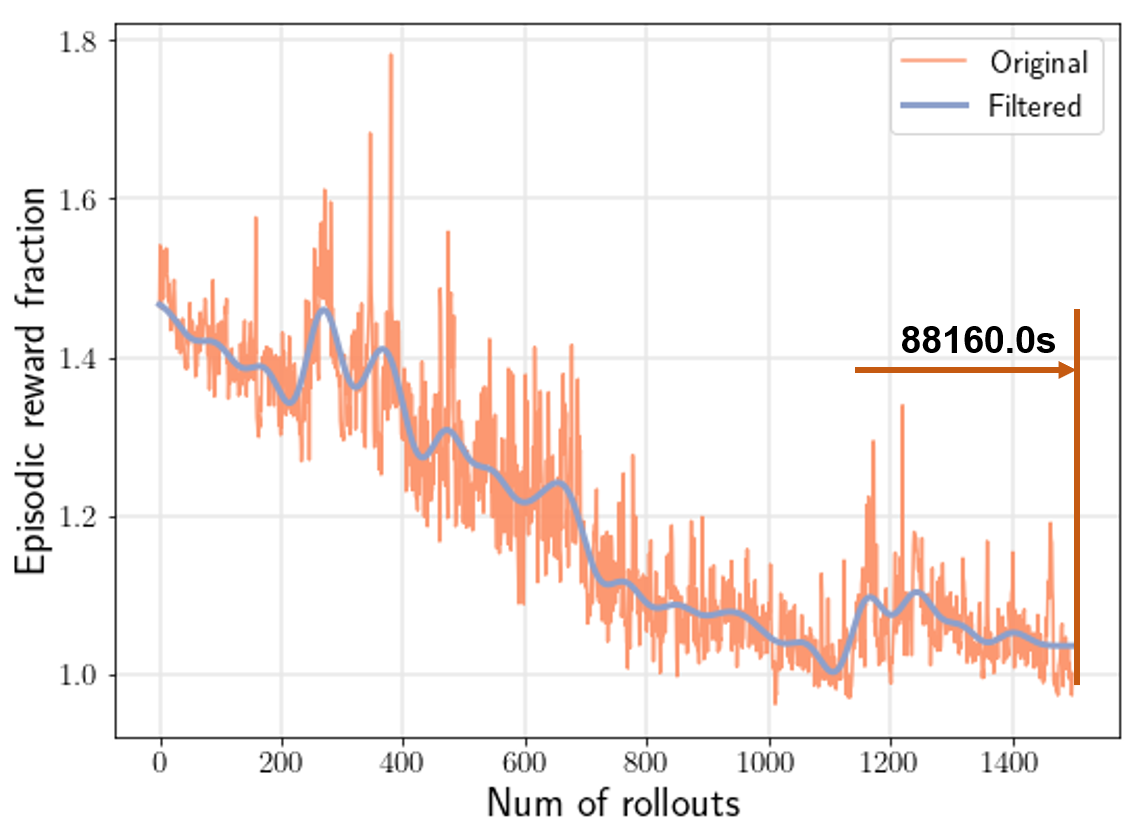}}
      \subfloat[Fish DDPG ]{\includegraphics[width=1.02\linewidth]{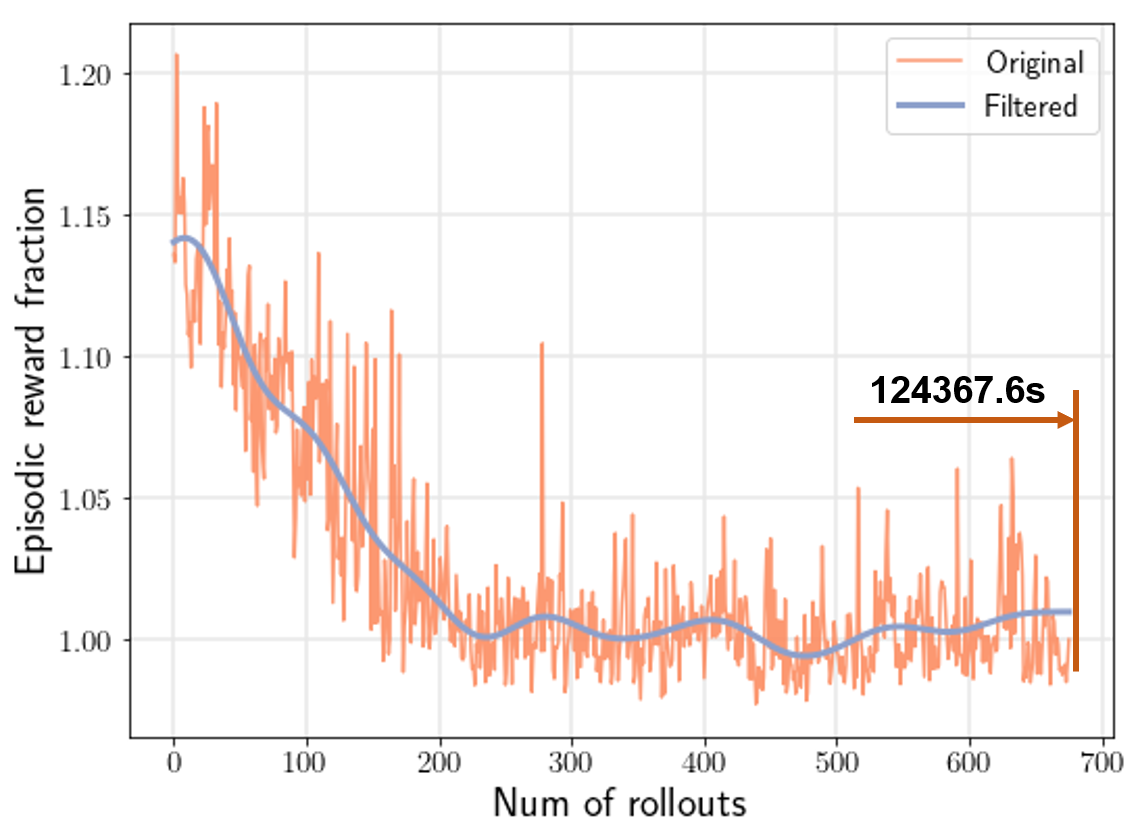}}
      \subfloat[Microstructure DDPG ]{\includegraphics[width=1.02\linewidth]{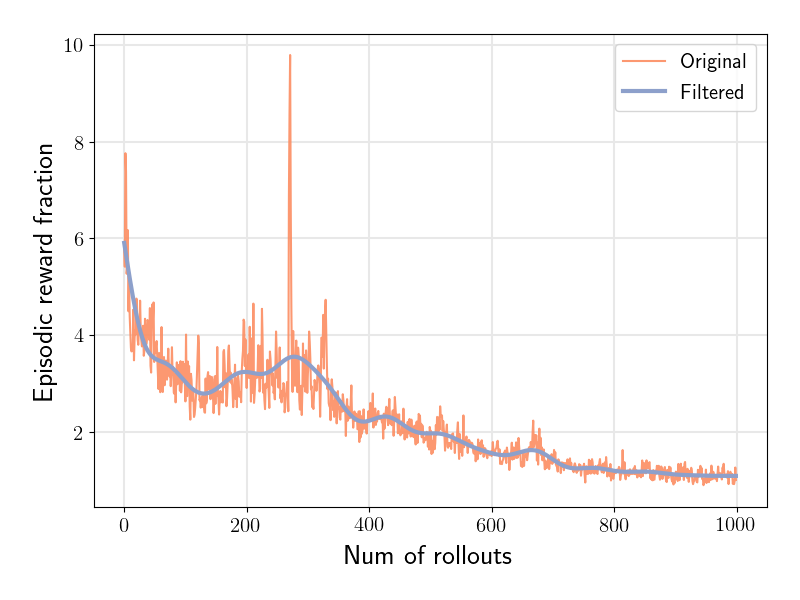}}
\end{multicols}
\caption{Episodic reward/cost fraction vs number of rollouts taken during training}
\label{d2c_training_rollout}
\end{figure*}

\begin{table*}[!htp]
\caption{D2C parameters}
\label{d2cparam}
\centering
\vspace{0.1in}
\begin{threeparttable}
\begin{tabular}{|c|c|c|c|c|c|}
\hline
System& Std of & LLS iteration & \multicolumn{3}{|c|}{Cost parameters\tnote{*}}\\\cline{4-6}
& LLS noise &number&&&\\
&&&$Q$&$Q_N$&$R$\\
\hline
Inverted&$10^{-6}$  &50&(1, 0.1)&100&1\\
Pendulum& & &&&\\
\hline
Cart pole& $10^{-5}$ &60& 0.1&5000&0.5\\
\hline
3-link & $10^{-5}$&30 & (8, 8, &(6000, 6000, &0.05\\
Swimmer& & &0, ...0) &0, ...0) &\\
\hline
6-link &$10^{-5}$&30 &(8, 8, &(1000, 1000, &0.001\\
Swimmer &&&0, ...0) &0, ...0) & \\ 
\hline
Fish&$10^{-5}$ &40 & (20, 20, 20, &(20000, 20000, 20000, &0.06\\
 &&&1, 0, ...0) &3000, 0, ...0) & \\
\hline
Material Microstructure&0.1 &- & 9&9000&0.1\\
\hline
\end{tabular}
\vspace{0.1in}
\begin{tablenotes}
        \footnotesize
        \item[*] $Q$ is the incremental cost matrix, $Q_N$ is the terminal cost matrix and R is the control cost matrix, all of which are diagonal matrices. If the diagonal elements have the same value, only one of them is presented in the table, otherwise all diagonal values are presented.
\end{tablenotes}
\end{threeparttable}
\end{table*}

\subsection{D2C and DDPG Implementation Details}
\label{sec:ddpg}
The parameters in the D2C implementation are summarized in Table \ref{d2cparam}. The iLQR algorithm searches over $\alpha$ by $\alpha = 0.99 \times \alpha$. The process noise variance $W = I$ and measurement noise variance $V = 0.01 \times I$ for all cases.

Deep Deterministic Policy Gradient (DDPG) is a policy-gradient based off-policy reinforcement learning algorithm that operates in continuous state and action spaces. It relies on two function approximation networks one each for the actor and the critic. The critic network estimates the $Q(s, a)$ value given the state and the action taken, while the actor network engenders a policy given the current state. Neural networks are employed to represent the networks. 

The off-policy characteristic of the algorithm employs a separate behavioral policy by introducing additive noise to the policy output obtained from the actor network. The critic network minimizes loss based on the temporal-difference (TD) error and the actor network uses the deterministic policy gradient theorem to update its policy gradient as shown below:

Critic update by minimizing the loss: 
\begin{equation*}
    L = \frac{1}{N} \Sigma_{i}(y_i - Q(s_i, a_i| \theta^{Q}))^2
\end{equation*}

Actor policy gradient update:
\begin{equation*}
    \nabla_{\theta^{\mu}} \approx \frac{1}{N} \Sigma_i \nabla_a Q(s,a|\theta^{Q})|_{s=s_i,a=\mu(s_i)} \nabla_{\theta^{\mu}} \mu(s|\theta^{\mu})|_{s_i}
\end{equation*}

The actor and the critic networks consist of two hidden layers with the first layer containing 400 ('{\it relu}' activated) units followed by the second layer containing 300 ('{\it relu}' activated) units. The output layer of the actor network has the number of ('tanh' activated) units equal to that of the number of actions in the action space.  

Target networks one each for the actor and the critic are employed for a gradual update of network parameters, thereby reducing the oscillations and better training stability. The target networks are updated at $\tau = 0.001$. Experience replay is another technique that improves the stability of training by training the network with a batch of randomized data samples from its experience. We have used a batch size of 32 for the inverted pendulum and the cart pole examples, whereas it is 64 for the rest. Finally, the networks are compiled using Adams' optimizer with a learning rate of 0.001. 

To account for state-space exploration, the behavioral policy consists of an off-policy term arising from a random process. We obtain discrete samples from the Ornstein-Uhlenbeck (OU) process to generate noise as followed in the original DDPG method. The OU process has mean-reverting property and produces temporally correlated noise samples as follows:
\begin{equation*}
    dx_t = \Theta (\mu - x_t)dt + \sigma dW
\end{equation*}
where $\Theta$ indicates how fast the process reverts to mean, $\mu$ is the equilibrium or the mean value and $\sigma$ corresponds to the degree of volatility of the process. $\Theta$ is set to 0.15, $\mu$ to 0 and $\sigma$ is annealed from 0.35 to 0.05 over the training process. 

\subsection{Estimation of Hessians: Linear Least Squares by Central Difference (LLS-CD)}
\label{hession_estimation}
Using the same Taylor's expansions as described in Section \ref{sys_id_solve}, we obtain the following central difference equation:
$
    F(\bar{x}_t + \delta x_t, \bar{u}_t + \delta u_t) + F(\bar{x}_t - \delta x_t, \bar{u}_t - \delta  u_t) = 2 F(\bar{x}_t, \bar{u}_t) + \begin{bmatrix} \delta {x_t}' & \delta {u_t}' \end{bmatrix} \begin{bmatrix} F_{x_tx_t} & F_{x_tu_t} \\ F_{u_tx_t} & F_{u_tu_t} \end{bmatrix} \begin{bmatrix}  \delta {x_t} \\ \delta {u_t} \end{bmatrix} + O(\| \delta {x_t}\|^4 + \| \delta {u_t}\|^4),
$
where $F_{x_tx_t} = \left.\dfrac{\partial^2 F}{\partial x^2}\right|_{x_t}$, similar for $F_{u_tx_t}$ and $F_{u_tu_t}$. Denote $z_t = F(\bar{x}_t + \delta x_t, \bar{u}_t + \delta u_t) + F(\bar{x}_t - \delta x_t, \bar{u}_t - \delta  u_t) - 2 F(\bar{x}_t, \bar{u}_t)$. The Hessian is a $(n_s + n_u)$ by $n_s$ by $(n_s + n_u)$ tensor, where $n_s$ is the number of states and $n_u$ is the number of actuators. Let's seperate the tensor into 2D matrices w.r.t. the second dimension and neglect time $t$ for simplicity of notations:

\begin{align}
z_i &= \begin{bmatrix} \delta {x}' & \delta {u}' \end{bmatrix} \begin{bmatrix} F_{xx}^{(i)}  & F_{xu}^{(i)}  \\ F_{ux}^{(i)}  & F_{uu}^{(i)}  \end{bmatrix} \begin{bmatrix}  \delta {x} \\ \delta {u} \end{bmatrix} \nonumber \\  \nonumber
&= \sum_{j=1}^{n_s} \sum_{k=1}^{n_s} \dfrac{\partial^2 F_i}{\partial x_j \partial x_k} \delta x_j \delta x_k + 2 \sum_{j=1}^{n_u} \sum_{k=1}^{n_s} \dfrac{\partial^2 F_i}{\partial u_j \partial x_k} \delta u_j \delta x_k + \sum_{j=1}^{n_u} \sum_{k=1}^{n_u} \dfrac{\partial^2 F_i}{\partial u_j \partial u_k} \delta u_j \delta u_k \\ \nonumber
&= \underbrace{\begin{bmatrix} \delta {x_1}^2 & \delta {x_1} \delta {x_2} & \cdots & \delta {x_1} \delta {u_{n_u}} & \delta {x_2}^2 & \delta {x_2} \delta {x_3} & \cdots& \delta {u^2_{n_u}}\end{bmatrix}}_{\delta {M}}  \underbrace{\begin{bmatrix} \dfrac{\partial^2 F_i}{\partial x_1^2} \\ \\ 2 \dfrac{\partial^2 F_i}{\partial x_1 \partial x_2} \\ \vdots \\ 2 \dfrac{\partial^2 F_i}{\partial x_1 \partial u_{n_u}} \\ \\ \dfrac{\partial^2 F_i}{\partial x_2^2} \\ \\ 2 \dfrac{\partial^2 F_i}{\partial x_2 \partial x_3} \\ \vdots \\ \dfrac{\partial^2 F_i}{\partial {u^2_{n_u}}}\end{bmatrix}}_{H_i},\\
\end{align}

where $F_{xx}^{(i)} = \left.\dfrac{\partial^2 F_i}{\partial x^2}\right|_{x_t}$, $z_i$ is the $i^{th}$ element of vector $z_t$ and $F_i$ is the $i^{th}$ element of the dynamics vector $F(x_t, u_t)$. Multiplying on both sides by $\delta M'$ and apply standard Least Square method:
$H_i = (\delta M^T \delta M)^{-1} \delta M^T z_i$. Then repeat for $i=1,2,...,n_s$ to get the estimation for the Hessian tensor.

\end{document}